\documentclass[10pt,twocolumn,letterpaper]{article}

\usepackage[pagenumbers]{cvpr} % To force page numbers, e.g. for an arXiv version

\definecolor{cvprblue}{rgb}{0.21,0.49,0.74}
\usepackage[pagebackref,breaklinks,colorlinks,allcolors=cvprblue]{hyperref}

\usepackage{booktabs}
\usepackage{arydshln}
\usepackage{amssymb}
\usepackage{multirow}
\usepackage{pifont}
\usepackage{subcaption}
\usepackage{amsmath}
\usepackage{amsthm}
\usepackage{algorithm}
\usepackage{algorithmic}
\usepackage{xcolor}
\usepackage{colortbl} % add at document top if not present
\newtheorem{theorem}{Theorem} % 定理环境，注意顺序

\usepackage{xspace}         % Correctly handle spaces after commands

\def\paperID{4184} % *** Enter the Paper ID here
\def\confName{CVPR}
\def\confYear{2026}

\title{SelecTKD: Selective Token-Weighted Knowledge Distillation for LLMs}

\author{Haiduo Huang\thanks{{\tt\small huanghd@stu.xjtu.edu.cn}},\; Jiangcheng Song,\; Yadong Zhang,\; Pengju Ren\\
Institute of Artificial Intelligence and Robotics,\; Xi'an Jiaotong University\\
}

\begin{document}
\maketitle
\begin{abstract}
    Knowledge distillation (KD) is a standard route to compress Large Language Models (LLMs) into compact students, yet most pipelines uniformly apply token-wise loss regardless of teacher confidence. This indiscriminate supervision amplifies noisy, high-entropy signals and is especially harmful under large teacher-student capacity gaps. We introduce SelecTKD, a plug-and-play \emph{Selective Token-Weighted} distillation framework that shifts the focus from ``how to measure divergence'' to ``where to apply learning''. At each step, the student proposes tokens that are verified by the teacher through a robust \emph{propose-and-verify} procedure with two variants: \emph{greedy Top-$k$} and \emph{non-greedy Spec-$k$}. Accepted tokens receive full loss, while rejected tokens are masked or down-weighted. This objective-agnostic design works with on- and off-policy data, induces an implicit curriculum quantified by Token Acceptance Rate (TAR), and stabilizes optimization. Across instruction following, mathematical reasoning, code generation, and a VLM setting, SelecTKD consistently improves strong baselines and achieves state-of-the-art results for small models—without architectural changes or extra reference models.
\end{abstract}

\section{Introduction}
\label{sec:intro}
Large Language Models (LLMs) have made rapid progress through scaling compute, parameters, and data~\cite{Kaplan2020,hurst2024gpt}, yet their inference-time latency and memory footprint hinder broad deployment~\cite{Xu2024}. Compressing LLMs into compact students without sacrificing capability is thus a central practical challenge. Knowledge distillation (KD)~\cite{Hinton2015} is the de facto approach for transferring competency from a powerful teacher to a smaller student and has enabled strong Small Language Models (SLMs), as illustrated by recent open-source families~\cite{team2025gemma,yang2025qwen3}. However, most KD pipelines apply a uniform token-wise loss, even when the teacher is uncertain, forcing the student to imitate high-entropy predictions. This indiscriminate supervision injects noise, especially under large capacity gaps, and can destabilize optimization.

\begin{figure*}[ht]
  \begin{center}
  \includegraphics[width=0.75\linewidth]{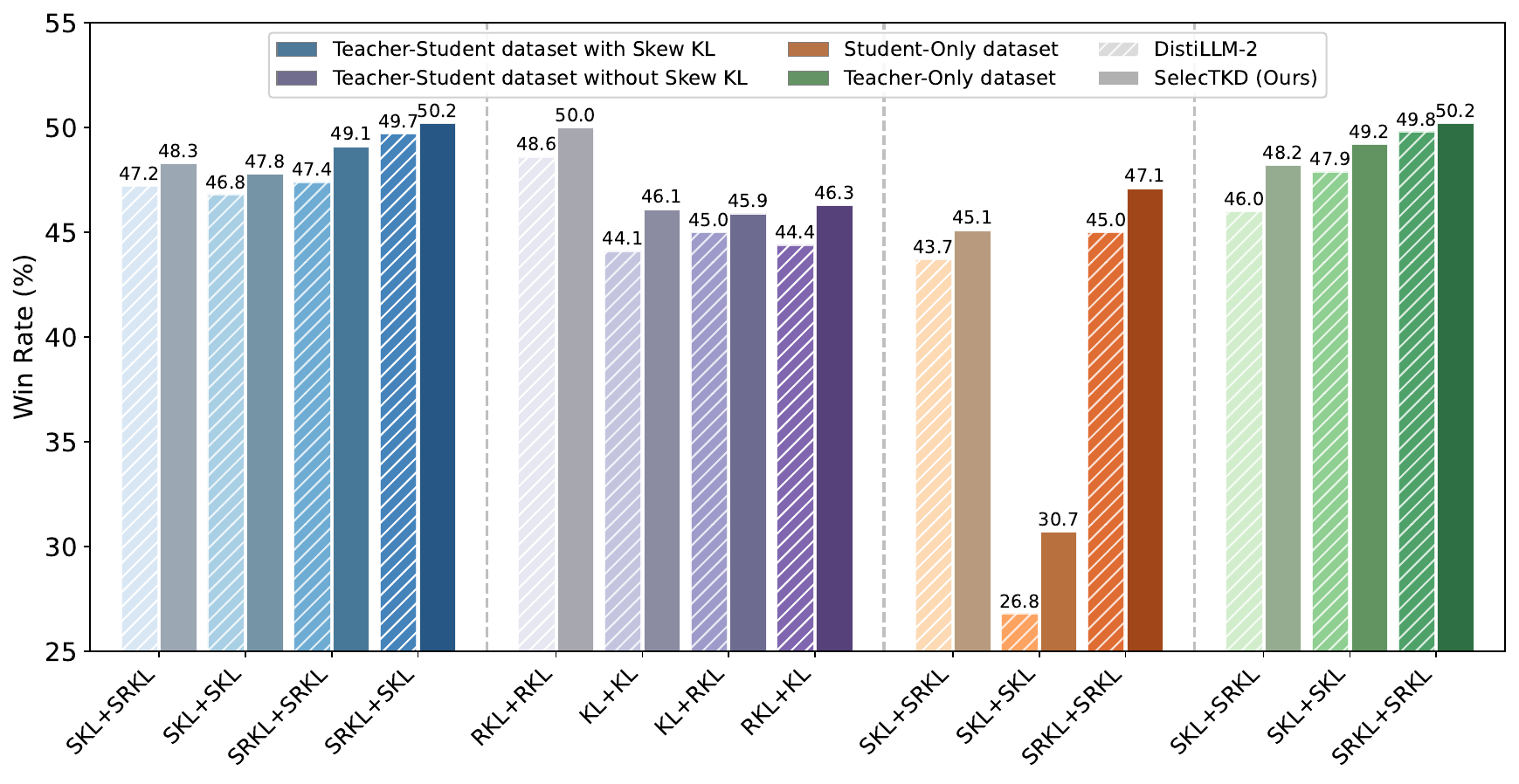}
  \caption{
    Performance comparison of different loss functions and training datasets on DistiLLM-2~\cite{Ko2025} and DistiLLM-2 with SelecTKD, using Qwen2-7B-Inst as the teacher and Qwen2-1.5B as the student. The chart shows Win Rate on the Evol-Instruct benchmark evaluated with GPT-4o. The baseline is gpt-3.5-turbo. This comprehensive comparison highlights the impact of various loss/data combinations and demonstrates the consistent advantage of our SelecTKD method.
  }
  \label{diff_data_loss_performance_comparison}
  \end{center}
  \vspace{-0.5cm}
\end{figure*} 

The literature on LLM distillation mainly spans three axes. (1) Data regimes: dynamic selection and on-/off-policy pipelines mitigate exposure bias~\cite{Zhou2023b,arora2022exposure,Lin2020,Agarwal2024,Ko2024}. (2) Objectives: beyond vanilla KL, reverse/adaptive/skewed divergences and other $f$-divergences are explored~\cite{Gu2023,wen2023f,Wu2024,Ko2024,Jung2025}. (3) Transferred knowledge: multi-granularity semantics and adaptive interpolation address what and when to transfer~\cite{Liu2022d,Wei2024a,Liu2024q,Shing2025}. DistiLLM-2~\cite{Ko2025} further argues that matching losses to data type (e.g., SKL for teacher data, SRKL for student data) is crucial. Yet we ask: is loss geometry the dominant factor?

Motivated by this question, we revisit the role of loss geometry from both empirical and theoretical perspectives. Our preliminary experiments (Figure~\ref{diff_data_loss_performance_comparison}) compare multiple KL-family objectives and loss/data pairings under identical training budgets and data curation, and observe surprisingly similar end performance across choices, even when DistiLLM-2-style asymmetric pairings are used; additional results on other evaluation platforms are provided in the Appendix. This trend echoes the analysis of \cite{Wu2024}: while forward/reverse/skewed KLs can yield different optimization \emph{paths}, they share the same fixed point in the limit (see also our Appendix proofs). In other words, much of the observed difference is attributable to training dynamics rather than the final solution.

This reframes the central question: rather than optimizing only \emph{how} to measure divergence, we should ask {\bf which tokens are worth learning from}. When the teacher-student capacity gap is large~\cite{Shing2025}, uniformly forcing the student to mimic all teacher predictions, including high-entropy or uncertain ones, injects noise and can degrade generalization. We therefore propose to control \emph{where and when} the learning signal is applied at the token level. Inspired by speculative decoding~\cite{leviathan2023fast,chen2023accelerating}, prior works have explored related ideas: SLIM~\cite{Raman2023} reweights losses but still optimizes logits and may lose dark knowledge~\cite{Zhao2022}; SKD~\cite{Xu2024a} improves data quality by interleaved sampling yet modifies the dataset rather than the supervision signal; AdaSPEC~\cite{Hu2025d} filters difficult tokens via a reference model at extra cost. 

In contrast, our {\bf SelecTKD} uses a simple \emph{propose-and-verify} mechanism with token-weighted loss to focus on high-confidence teacher signals. It is plug-and-play, objective-agnostic (KL/RKL/SKL/SRKL), compatible with on- and off-policy regimes, and induces a stable implicit curriculum quantified by Token Acceptance Rate (TAR). Concretely, the student proposes tokens that the teacher verifies via two variants—\emph{greedy Top-$k$} and \emph{non-greedy Spec-$k$}; accepted tokens receive full loss, while rejected tokens are masked or weakly weighted.

Our main {\bf contributions} include: 
(i) We reframe KD from ``how to measure divergence'' to \textit{``where to apply supervision''}, introducing a selective token-weighted mechanism that emphasizes teacher-consistent tokens. 
(ii) We present a robust verification scheme—\emph{greedy Top-$k$} and \emph{non-greedy Spec-$k$}—with a single intuitive hyperparameter $k$ and a rejected-token weight $\beta$.
(iii) We provide theoretical support showing monotonic improvement of TAR and discuss its link to smoother loss landscapes and better generalization. 
(iv) We demonstrate consistent gains across instruction following, math, code, and a VLM setting—achieving SOTA results for small models—while remaining plug-and-play and compatible with KL, RKL, SKL, and SRKL.

\section{Related Work}
\noindent{\bf Knowledge Distillation for LLMs.}
Knowledge distillation (KD) \cite{Hinton2015} compresses LLMs by transferring a strong teacher to a smaller student. Recent advances cluster into two threads: \emph{objectives} and \emph{teaching strategies}. On objectives, MiniLLM \cite{Gu2023} leverages reverse KL, AKL \cite{Wu2024} adaptively blends forward/reverse KL, DistiLLM \cite{Ko2024} introduces skew KL, and DistiLLM-2 \cite{Ko2025} extends it with a contrastive, asymmetric design. On strategies, ImitKD \cite{Lin2020} mitigates exposure bias via on-policy imitation, PromptKD \cite{Kim2024} elicits student-friendly teacher outputs through prompts, and ATKD \cite{Zhong2024} adjusts teaching modes for easy vs. hard tokens under uncertainty. These works primarily optimize how divergence is measured and how data is curated. Our SelecTKD decides \emph{where} to apply supervision by selecting reliable tokens, and composes with the above losses and pipelines to further boost student performance.

\noindent{\bf Token Filtering and Selective Distillation.}
Recent work converges on the view that not all tokens are equally informative. Along the efficiency axis, token dropping in pretraining reduces computation by skipping low-importance tokens~\cite{hou2022token,zhong2023revisiting}. For fine-grained supervision, RHO-1~\cite{Lin2024} formalizes selective language modeling by reweighting losses across tokens, while ENT~\cite{Li2023v} improves robustness by truncating gradients of high-error tokens under noisy supervision. In the fine-tuning and alignment stages, Token Cleaning~\cite{Pang2025} and T-SHIRT~\cite{Fu2025a} apply token-level data selection that outperforms coarse, sample-level filtering; system-level advances such as Collider~\cite{Chai2025} further make token filtering cost-effective on modern hardware.  In the context of distillation, AdaSPEC~\cite{Hu2025d} performs selective alignment between draft and target models, optimizing acceptance rates instead of full-sequence KL. Our approach is complementary: we adopt a simple verification mechanism and apply a token-weighted loss directly during KD, emphasizing high-confidence teacher signals while preserving compatibility with existing objectives and pipelines.

\section{Preliminaries}
\label{subsec:preliminaries}
We formalize knowledge distillation (KD) for autoregressive LLMs and set up token-level divergences used throughout. An autoregressive model generates an output sequence {\small $\boldsymbol{y} = [y_1,\dots,y_{|\boldsymbol{y}|}]$} conditioned on an input {\small $\boldsymbol{x}$}. At step {\small $t$}, it predicts a token from a finite vocabulary {\small $\mathcal{V} = \{v_1,\dots,v_{|\mathcal{V}|}\}$}. The teacher and student define conditional distributions {\small $p(y_t\mid \boldsymbol{x}, \boldsymbol{y}_{<t})$} and {\small $q_\theta(y_t\mid \boldsymbol{x}, \boldsymbol{y}_{<t})$}, respectively.

\noindent{\bf Token-level divergences.}
At each time step {\small $t$} and for each vocabulary token {\small $v_i\in\mathcal{V}$}, we write {\small $p_i := p(v_i\mid \boldsymbol{x}, \boldsymbol{y}_{<t})$} and {\small $q_i := q_\theta(v_i\mid \boldsymbol{x}, \boldsymbol{y}_{<t})$}. The forward and reverse KL contributions at {\small $(t,i)$} are
{\small \begin{equation} \label{eq:fkl_def} D_{\text{FKL}}^{(t,i)}(p, q_\theta) = p_i \, \log \frac{p_i}{q_i}, \qquad D_{\text{RKL}}^{(t,i)}(p, q_\theta) = q_i \, \log \frac{q_i}{p_i}. \end{equation} }

\noindent{\bf Sequence-level objectives.}
Aggregating over steps and vocabulary yields the total forward and reverse KL losses:
{\small
\begin{equation}
\label{eq:fkl_rkl_total}
\mathcal{L}_{\mathrm{FKL}} = \sum_{t=1}^{|\boldsymbol{y}|} \sum_{i=1}^{|\mathcal{V}|} D_{\mathrm{FKL}}^{(t,i)}(p, q_\theta),\:
\mathcal{L}_{\mathrm{RKL}} = \sum_{t=1}^{|\boldsymbol{y}|} \sum_{i=1}^{|\mathcal{V}|} D_{\mathrm{RKL}}^{(t,i)}(p, q_\theta).
\end{equation}
}However, forward KL encourages mass-covering, whereas reverse KL favors mode-seeking~\cite{Gu2023}. To balance these, DistiLLM~\cite{Ko2024} proposed skewed variants, skew KL (SKL) and skew reverse KL (SRKL), which interpolate between teacher and student distributions.
{\small
\begin{align}
    &D_{\text{SKL}}^{(\alpha)}(\boldsymbol{x}, \boldsymbol{y}; p \| q_\theta) = D_{\text{KL}}(\boldsymbol{x}, \boldsymbol{y}; p \| \alpha p + (1-\alpha) q_\theta), \\
    &D_{\text{SRKL}}^{(\alpha)}(\boldsymbol{x}, \boldsymbol{y}; p \| q_\theta) = D_{\text{KL}}(\boldsymbol{x}, \boldsymbol{y}; q_\theta \| (1-\alpha) p + \alpha q_\theta).
\end{align}
}where {\small $\alpha\in[0,1]$} controls the mixing ratio. Building on this, DistiLLM-2~\cite{Ko2025} applies SKL to teacher-generated data and SRKL to student-generated data with a curriculum over skew coefficients. Let {\small $D^{(t)}$} and {\small $D^{(s)}$} denote batches of teacher and student responses. The per-step objective is
{\small
\begin{equation}
\label{eq:distillm2_prelim}
\begin{aligned}
\mathcal{L}_{\text{DistiLLM-2}} 
&= (1-\mu)\, \mathbb{E}_{(\boldsymbol{x}, \boldsymbol{y})\in D^{(t)}} \!\big[D_{\mathrm{SKL}}^{(\alpha_t)}\!(\boldsymbol{x}, \boldsymbol{y}; p\,\|\,q_\theta) \big] \\
&\quad +\; \mu\, \mathbb{E}_{(\boldsymbol{x},\boldsymbol{y})\in D^{(s)}} \!\big[D_{\mathrm{SRKL}}^{(\alpha_s)}\!(\boldsymbol{x}, \boldsymbol{y}; p\,\|\,q_\theta) \big],
\end{aligned}
\end{equation}
}where {\small $\mu \in[0,1]$} trades off teacher vs. student distributions and {\small $\alpha_t, \alpha_s$} are scheduled over training. All objectives are implemented token-wise.

\section{Method}
In this section, we present {\bf SelecTKD}, our proposed framework designed to address the limitations of uniform, indiscriminate loss application in conventional knowledge distillation, as shown in Figure~\ref{SelecTKD:framework}. Drawing inspiration from speculative decoding, SelecTKD introduces a dynamic, token-level verification mechanism that selectively filters the distillation signal, focusing learning on tokens where the student and teacher are sufficiently aligned.

\begin{figure*}[ht]
  \centering
  \includegraphics[width=0.85\linewidth]{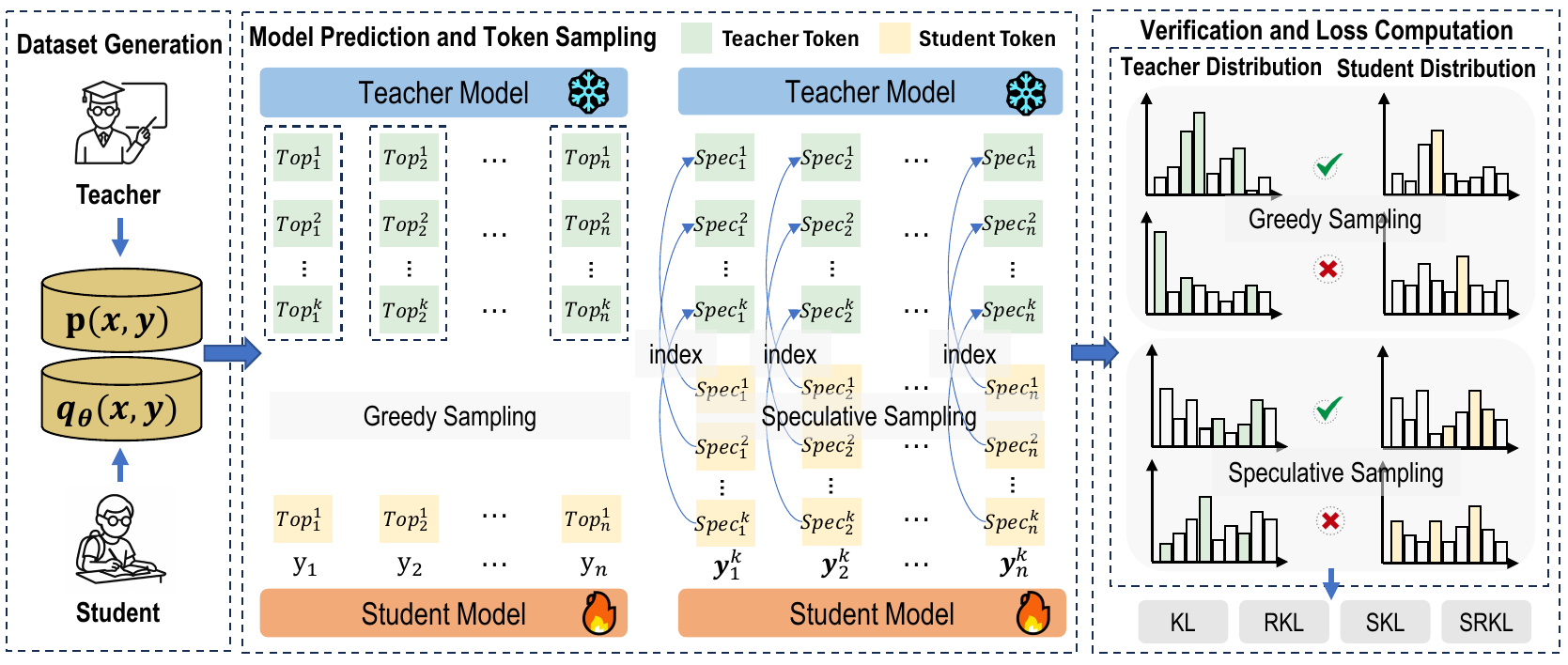}
  \caption{
    Overview of SelecTKD, a ``propose-and-verify'' framework for selective token-level knowledge distillation. 
    Left: {\bf Dataset Generation}. Both teacher $p(\boldsymbol{x},\boldsymbol{y})$ and student $q_\theta(\boldsymbol{x},\boldsymbol{y})$ can produce responses, enabling on- and off-policy data. 
    Middle: {\bf Model Prediction and Token Sampling}. Two variants are supported: 
    (a) \textit{Greedy Top-$k$}: the student proposes a greedy token $\hat y_t$ (argmax) and the teacher returns its Top-$k$ set; 
    (b) \textit{Non-greedy Spec-$k$}: the student samples $k$ candidate tokens and the teacher is queried only on these candidates to compute acceptance indices following speculative sampling~\cite{leviathan2023fast,chen2023accelerating}. 
    Right: {\bf Verification and Loss Computation}. A token is accepted if it is in the teacher's Top-$k$ (greedy) or passes the speculative acceptance test (non-greedy). Accepted tokens receive full loss; rejected tokens are masked or down-weighted by $\beta$. The loss is objective-agnostic and works with KL, RKL, SKL, and SRKL. This design focuses learning on reliable, teacher-aligned tokens and yields a stable, curriculum-like training dynamic.
  }
  \label{SelecTKD:framework}
  \vspace{-0.5cm}
\end{figure*}

Formally, the SelecTKD loss is defined with a generic token-wise divergence {\small $D(\cdot\|\cdot)$} (e.g., KL/RKL/SKL/SRKL):
{\small
\begin{equation}
  \label{eq:selec_tkd_def}
  \mathcal{L}_{\text{SelecTKD}} = \sum_{t=1}^{|\boldsymbol{y}|} V_t \, D\big(p_t \big\| q_t\big)
\end{equation}
}where {\small $p_t = p(\cdot\mid\boldsymbol{x}, \boldsymbol{y}_{<t})$} and {\small $q_t = q_\theta(\cdot\mid\boldsymbol{x}, \boldsymbol{y}_{<t})$} are the teacher and student distributions at position {\small $t$}, and {\small $V_t \in \{0, \beta, 1\}$} is a token-level verification weight determined by a verifier {\small $V(\cdot)$}\footnote{Note, $V(\cdot)$ will pass a detach operation to implement the stop-gradient operator.}. The core innovation of SelecTKD lies in {\small $V_t$}, which adaptively decides whether to apply the distillation loss at each token based on student-teacher alignment.

\subsection{Verification Mechanisms in SelecTKD}
\subsubsection{Hellinger-Based Verification}
An intuitive starting point is to compare the teacher and student probability distributions with the Hellinger distance~\cite{gonzalez2013class}. For a step $t$, let $p_t$ and $q_t$ denote the teacher and student distributions on $\mathcal{V}$. The Hellinger distance is
{\small
\begin{equation}
\label{eq:hellinger_def}
H(p_t,q_t) = \frac{1}{\sqrt{2}}\,\big\|\sqrt{p_t}-\sqrt{q_t}\big\|_2,\quad 0\le H\le 1.
\end{equation}
}It is symmetric and bounded, offering numerical stability. A soft instantiation sets {\small $V_t{=}H(p_t,q_t)$} (``more similar, less learning''), but Hellinger is tail-sensitive and rank-agnostic—small low-probability shifts can change $H$ despite agreement on Top-$k$, misaligning supervision. We therefore adopt a {\bf discrete, rank-based} Top-$k$ verifier that focuses on high-confidence tokens, ignores tail noise and calibration bias, and is computationally simpler. This choice also enables the {\bf Spec-$k$} variant and aligns with our observations of rising TAR and flatter loss landscapes.

\subsubsection{Greedy Top-\texorpdfstring{$k$}{k} Token Verification}
To enhance robustness, we propose a discrete, rank-based verification mechanism. At each step $t$, the student proposes its most likely token:
{\small
\begin{equation}
\hat{y}_t = \arg\max_{y} \, q_\theta(y\mid\boldsymbol{x}, \boldsymbol{y}_{<t})
\end{equation}
}We then check whether this token is among the teacher's top-$k$ candidates:
{\small
\begin{equation}
V_t \;=\; \beta \;+\; (1-\beta)\,\mathbb{I}\!\left( \hat{y}_t \in \mathrm{Top}_k(p_t) \right)
\end{equation}
}where $\mathrm{Top}_k(p_t)$ denotes the set of $k$ tokens with the highest probabilities under the teacher's distribution at position $t$. When $\beta=0$, this reduces to hard gating; small $\beta>0$ provides gentle regularization. This method is robust to the absolute probability values and provides an intuitive knob ($k$) to control the strictness of verification.

\subsubsection{Non-Greedy Spec-\texorpdfstring{$k$}{k} Verification}
To further mitigate exposure bias and align with speculative sampling~\cite{leviathan2023fast}, we verify \emph{multiple} student proposals per step with a rejection-sampling style test.
At each position $t$, the student draws $k$ i.i.d. candidates from its own distribution:  
{\small
\begin{equation}
\tilde{\mathcal{Y}}_t=\big\{y^{(1)}_t,\dots,y^{(k)}_t\big\}\sim \text{i.i.d. } q_\theta(\cdot\mid\boldsymbol{x},\boldsymbol{y}_{<t}).
\end{equation}
}For each sampled token $y^{(j)}_t$, we only query the token and compute the standard speculative-acceptance probability
{\small
\begin{equation}
\label{eq:spec_accept_prob}
a_t^{(j)} \;=\; \min\!\left(1,\; \frac{p_t\!\big(y^{(j)}_t\big)}{q_t\!\big(y^{(j)}_t\big)}\right)
\end{equation}
}Draw $r^{(j)}\!\sim\!\mathrm{U}[0,1]$ and accept the candidate if $r^{(j)}<a_t^{(j)}$.
Let ${\mathcal{A}}_t=\{j\in[1..k]\mid r^{(j)}<a_t^{(j)}\}$ be the set of accepted samples.
We then define the verification weight:
{\small
\begin{equation}
\label{eq:ng_topk_Vt}
V_t = \beta + (1-\beta)\, \mathbb{I}\!\left(|{\mathcal{A}}_t| \ge 1\right)
\end{equation}
}Intuitively, if any of the $k$ proposals are accepted by the teacher under the speculative test, the teacher's local signal is informative and the student should proceed with normal learning; if all are rejected, supervision at this step is likely noisy or misaligned and is therefore down-weighted. This criterion inherits the robustness of speculative sampling, requires $k$ probability lookups at the teacher to score the student's proposals, curriculum-like~\cite{liu2024let} behavior.

\subsection{Token Acceptance Rate (TAR)}
A key property of SelecTKD is its stable and convergent training dynamics, which can be quantified by the \textit{Token Acceptance Rate} (TAR):
{\small
\begin{equation}
\mathrm{TAR} = \mathbb{E}_{(\boldsymbol{x}, \boldsymbol{y})} \left[ \frac{1}{|\boldsymbol{y}|} \sum_{t=1}^{|\boldsymbol{y}|} \mathbb{I}\!\big(V_t=1\big) \right]
\end{equation}
}TAR measures the expected fraction of tokens for which the student's proposal is accepted by the teacher's verifier (i.e., full-weight tokens). We prove (see Appendix, Theorem 1) that, under mild assumptions, SelecTKD guarantees a monotonic increase in TAR during training, reflecting improved student-teacher alignment.

\noindent{\bf The SelecTKD Algorithm.}
We summarize the full SelecTKD training procedure in Algorithm~\ref{alg:SelecTKD}. The core is the dynamic, token-level verification that determines whether to apply the distillation loss at each position. Note that the validation operation for each token position can be batched, and the actual implementation does not require a for loop.

\begin{algorithm}\small
\caption{The training process of SelecTKD}
\label{alg:SelecTKD}
\begin{algorithmic}[1]
\STATE {\bf Input:} Teacher $p$, Student $q_\theta$ (params $\theta$), Top-$k$ or Spec-$k$ value $k$, rejected-token weight $\beta$, verification {\bf mode} $\in\{\texttt{greedy}, \texttt{non-greedy}\}$
\STATE {\bf Output:} Updated student $q_{\hat{\theta}}$
\FOR{each batch $(\boldsymbol{x}, \boldsymbol{y})$ in dataset $\mathcal{D}$}
    \STATE Initialize total loss $\mathcal{L}_{\text{total}} = 0$
    \FOR{each token position $t = 1, \ldots, |\boldsymbol{y}|$}
        \STATE $p_t \leftarrow p(\cdot|\boldsymbol{x}, \boldsymbol{y}_{<t})$, \quad $q_t \leftarrow q_\theta(\cdot|\boldsymbol{x}, \boldsymbol{y}_{<t})$
        \IF{{\bf mode} == \texttt{greedy}}
            \STATE $\hat{y}_t \leftarrow \arg\max_{y} q_t(y)$
            \STATE $V_t \leftarrow \beta + (1-\beta)\cdot \mathbb{I}\!\left(\hat{y}_t \in \mathrm{Top}_k(p_t)\right)$
        \ELSIF{{\bf mode} == \texttt{non-greedy}}
            \STATE Sample $k$ candidates $\{y^{(j)}_t\}_{j=1}^k \sim \text{i.i.d. } q_t$
            \STATE For each $j$: $a_t^{(j)} \leftarrow \min\!\left(1, \frac{p_t(y^{(j)}_t)}{q_t(y^{(j)}_t)}\right)$, draw $r^{(j)}\!\sim\!\mathrm{U}[0,1]$
            \STATE ${\mathcal{A}}_t \leftarrow \{j \in [1..k] \mid r^{(j)} < a_t^{(j)}\}$
            \STATE $V_t \leftarrow \beta + (1-\beta)\cdot \mathbb{I}\!\left(|{\mathcal{A}}_t| \ge 1\right)$
        \ENDIF
        \STATE $L_t \leftarrow V_t \cdot D(p_t \| q_t)$ \COMMENT{$D\in\{$KL, RKL, SKL, SRKL$\}$}
        \STATE $\mathcal{L}_{\text{total}} \leftarrow \mathcal{L}_{\text{total}} + L_t$
    \ENDFOR
    \STATE $\mathcal{L}_{\text{SelecTKD}} \leftarrow \mathcal{L}_{\text{total}} / |\boldsymbol{y}|$
    \STATE Update $\theta$ using $\nabla_\theta \mathcal{L}_{\text{SelecTKD}}$
\ENDFOR
\end{algorithmic}
\end{algorithm}

\section{Experiments}
\label{sec:experiments}
\subsection{Experiment Setup in LLMs}
All experiments use the Hugging Face TRL framework with LoRA~\cite{Ko2025} on 8 NVIDIA A100 80GB GPUs. Unless otherwise noted, SelecTKD uses $k = 5$ and $\beta = 0.01$. We run 2 epochs for instruction/code and 1 epoch for math, with early stopping on a held-out validation set. LoRA ranks and target modules are fixed across methods for fairness.\\
\noindent {\bf General Instruction-Following.}
We randomly sample 50k prompts from UltraChat200k~\cite{ding2023enhancing} to construct the training set. Three teacher-student pairs are evaluated: Qwen2-7B-Instruct~\cite{hui2024qwen2} $\rightarrow$ Qwen2-1.5B-Instruct, Mistral-7B-Instruct~\cite{MistralAI2023} $\rightarrow$ Danube2-1.8B~\cite{singer2024h2o}, and Gemma-2-9B-IT~\cite{team2024gemma} $\rightarrow$ Gemma-2-2B-IT. Performance is assessed on AlpacaEval~\cite{dubois2024length}, Evol-Instruct~\cite{xu2024wizardlm}, and UltraFeedback~\cite{cui2023ultrafeedback}, following the LLM-as-a-Judge~\cite{zheng2023judging} protocol with GPT-4o or GPT-4o-mini as judges to report win rates. We reserve 1k held-out prompts from the 50k for validation and use the remainder for training. For evaluation sets, we follow DistiLLM-2~\cite{Ko2025} splits and apply de-duplication against our training sources by prompt hashing.

\noindent {\bf Mathematical Reasoning.}
We evaluate on GSM8K~\cite{cobbe2021training} and MATH~\cite{hendrycks2021measuring} benchmarks using Qwen2-Math-7B-Inst and Qwen2.5-Math-7B-Inst as teachers, and Qwen2-Math-1.5B and Qwen2.5-Math-1.5B as students. Each student is first supervised fine-tuned on the full MetaMathQA dataset for one epoch, then distilled with 50k samples from MetaMathQA~\cite{yu2023metamath}. Pass@1 accuracy is used for evaluation.

\noindent{\bf Code Generation.}
For code generation, we use WizardCoder prompts (Evol-Instruct) for training. Teacher-student pairs include Qwen2.5-Coder-7B-Inst $\rightarrow$ Qwen2.5-Coder-1.5B and DeepSeek-Coder-6.7B-Inst $\rightarrow$ DeepSeek-Coder-1.3B. Students are trained for 2 epochs and evaluated on HumanEval and MBPP using Pass@1 accuracy.

\noindent {\bf Reporting.}
Following DistiLLM-2~\cite{Ko2025}, all win rates and accuracy scores are averaged over 3 seeds. We report 95\% confidence intervals from 10{,}000 bootstrap resamples stratified by task. For LLM-as-a-Judge, we use randomized, pairwise side-by-side comparisons and tie-aware scoring, excluding ties from win rate calculation. Judge prompts and system messages are fixed across all methods.

\noindent{\bf Baselines.}
We compare SelecTKD with a range of representative knowledge distillation methods:
{\bf Vanilla KD:} Standard supervised distillation on teacher token distributions.
{\bf SeqKD:} Supervised training on full teacher-generated sequences.
{\bf ImitKD:} Hybrid on-policy distillation using a mix of ground-truth and student outputs.
{\bf GKD:} On-policy distillation with student-generated data and alternative divergences.
{\bf DistiLLM:} Framework with skew KLD loss and adaptive off-policy reuse of student data.
{\bf SKD:} Speculative decoding-inspired method where the teacher corrects low-quality student tokens to build a better training set.
{\bf DistiLLM-2:} Contrastive method applying asymmetric losses (SKL/SRKL) to teacher and student data.
The training strategy for LLMs is consistent with DistillLLM-2~\cite{Ko2025}.

\subsection{Experiment Setup in VLMs}
For SelecTKD-VLM-2B (InternVL2-2B+SelecTKD), we trained on the Mantis-Instruct dataset~\cite{jiang2024mantis}. The training process needs only $3$ epochs, using the AdamW ($\beta _1$=0.9, $\beta _2$=0.95) optimizer with a cosine learning rate scheduler, batch size 128, no warmup. The initial learning rate was set to $8\mathrm{e}{-5}$. In this vision-language model distillation task, we employed InternVL2-8B~\cite{Chen2024} as the teacher model and InternVL2-2B as the student model.

\subsection{Main Results}
\subsubsection{General Instruction-Following in LLMs}
Table~\ref{tab:base} presents the main results on three instruction-following benchmarks. Our proposed SelecTKD framework consistently achieves the highest win rates across all tasks and model pairs, substantially outperforming Vanilla KD and other strong baselines. The improvements are especially pronounced in challenging, open-ended generation scenarios, underscoring the effectiveness of SelecTKD's selective loss filtering in mitigating noisy or uninformative learning signals. Notably, SelecTKD delivers clear gains even when applied on top of advanced distillation methods such as DistiLLM-2 and SKD, demonstrating its generality and robustness for aligning smaller student models with nuanced human preferences.

\begin{table}[htbp]
  \caption{Win rates (WR) on three instruction-following benchmarks. The baseline is \texttt{text-davinci-003} for AlpacaEval and \texttt{gpt-3.5-turbo} for Evol-Instruct and UltraFeedback. Judging is performed by GPT-4o (AlpacaEval, Evol-Instruct) and GPT-4o-mini (UltraFeedback). `\emph{+}Ours' denotes SelecTKD.}
  \label{tab:base}
  \centering
  \resizebox{1.0\linewidth}{!}{
  \begin{tabular}{lcccc}
    \toprule
    & {\bf AlpacaEval} & {\bf Evol-Instruct} & {\bf UltraFeedback} & {\bf AVG.} \\
    & {\bf WR(\%)} & {\bf WR(\%)} & {\bf WR(\%)} & {\bf WR(\%)} \\
    \midrule[0.12em]
    \multicolumn{5}{c}{{\bf Qwen2-7B-Inst ($\mathcal{M}_T$) $\rightarrow$ Qwen2-1.5B ($\mathcal{M}_S$)}} \\
    \hline
    $\mathcal{M}_T$         & 88.41 & 70.70 & 69.25 & 76.12 \\
    $\mathcal{M}_S$         & 51.06 & 18.00 & 21.93 & 30.33 \\
    \hline
    KD                      & 57.49 & 28.23 & 37.86 & 41.19 \\
    \emph{+}Ours            & 59.72 (\textcolor{red}{+2.23}) & 30.35 (\textcolor{red}{+2.12}) & 39.93 (\textcolor{red}{+2.07}) & 43.33 (\textcolor{red}{+2.14}) \\
    SeqKD                   & 58.02 & 29.11 & 38.35 & 41.83 \\
    \emph{+}Ours            & 60.16 (\textcolor{red}{+2.14}) & 31.18 (\textcolor{red}{+2.07}) & 40.77 (\textcolor{red}{+2.42}) & 44.04 (\textcolor{red}{+2.21}) \\
    ImitKD                  & 59.37 & 30.58 & 39.92 & 43.29 \\
    \emph{+}Ours            & 61.16 (\textcolor{red}{+1.79}) & 32.30 (\textcolor{red}{+1.72}) & 41.33 (\textcolor{red}{+1.41}) & 44.93 (\textcolor{red}{+1.64}) \\
    GKD                     & 66.07 & 44.61 & 57.74 & 56.14 \\
    \emph{+}Ours            & 66.95 (\textcolor{red}{+0.88}) & 45.73 (\textcolor{red}{+1.12}) & 58.61 (\textcolor{red}{+0.87}) & 57.10 (\textcolor{red}{+0.96}) \\
    DistiLLM                & 66.30 & 44.61 & 58.18 & 56.36 \\
    \emph{+}Ours            & 67.21 (\textcolor{red}{+0.91}) & 45.72 (\textcolor{red}{+1.11}) & 58.62 (\textcolor{red}{+0.44}) & 57.18 (\textcolor{red}{+0.82}) \\
    SKD                     & 61.52 & 44.95 & 56.82 & 54.43 \\
    \emph{+}Ours            & 63.19 (\textcolor{red}{+1.67}) & 45.97 (\textcolor{red}{+1.02}) & 57.96 (\textcolor{red}{+1.14}) & 55.71 (\textcolor{red}{+1.28}) \\
    DistiLLM-2              & 69.88 & 47.13 & 59.05 & 58.69 \\
    \emph{+}Ours            & {\bf 70.60} (\textcolor{red}{+0.72}) & {\bf 48.27} (\textcolor{red}{+1.14}) & {\bf 59.53} (\textcolor{red}{+0.48}) & {\bf 59.47} (\textcolor{red}{+0.78}) \\
    \midrule[0.12em]
    \multicolumn{5}{c}{{\bf Mistral-7B-Inst ($\mathcal{M}_T$) $\rightarrow$ Danube2-1.8B ($\mathcal{M}_S$)}} \\
    \hline
    $\mathcal{M}_T$         & 91.92 & 73.51 & 83.59 & 83.01 \\
    $\mathcal{M}_S$         & 48.17 & 12.84 & 20.06 & 27.02 \\
    \hline
    KD                      & 60.21 & 18.23 & 41.56 & 40.00 \\
    \emph{+}Ours            & 62.56 (\textcolor{red}{+2.35}) & 20.21 (\textcolor{red}{+1.98}) & 43.60 (\textcolor{red}{+2.04}) & 42.12 (\textcolor{red}{+2.12}) \\
    SeqKD                   & 59.76 & 18.45 & 42.11 & 40.11 \\
    \emph{+}Ours            & 62.31 (\textcolor{red}{+2.55}) & 20.44 (\textcolor{red}{+1.99}) & 43.89 (\textcolor{red}{+1.78}) & 42.21 (\textcolor{red}{+2.10}) \\
    ImitKD                  & 58.34 & 17.89 & 40.87 & 39.03 \\
    \emph{+}Ours            & 60.90 (\textcolor{red}{+2.56}) & 20.54 (\textcolor{red}{+2.65}) & 43.45 (\textcolor{red}{+2.58}) & 41.63 (\textcolor{red}{+2.60}) \\
    GKD                     & 69.75 & 24.54 & 57.74 & 50.68 \\
    \emph{+}Ours            & 70.40 (\textcolor{red}{+0.65}) & 25.82 (\textcolor{red}{+1.28}) & 58.33 (\textcolor{red}{+0.59}) & 51.52 (\textcolor{red}{+0.84}) \\
    DistiLLM                & 70.16 & 28.78 & 58.18 & 52.37 \\
    \emph{+}Ours            & 71.05 (\textcolor{red}{+0.89}) & 29.95 (\textcolor{red}{+1.17}) & 59.11 (\textcolor{red}{+0.93}) & 53.37 (\textcolor{red}{+1.00}) \\
    SKD                     & 64.58 & 38.87 & 60.04 & 54.50 \\
    \emph{+}Ours            & 66.56 (\textcolor{red}{+1.98}) & {\bf 38.96} (\textcolor{red}{+0.09}) & 60.95 (\textcolor{red}{+0.91}) & 55.49 (\textcolor{red}{+0.99}) \\
    DistiLLM-2              & 74.04 & 32.84 & 62.46 & 56.45 \\
    \emph{+}Ours            & {\bf 74.69} (\textcolor{red}{+0.65}) & 33.78 (\textcolor{red}{+0.94}) & {\bf 63.03} (\textcolor{red}{+0.57}) & {\bf 57.17} (\textcolor{red}{+0.72}) \\
    \midrule[0.12em]
    \multicolumn{5}{c}{{\bf Gemma-2-9B-Inst ($\mathcal{M}_T$) $\rightarrow$ Gemma-2-2B ($\mathcal{M}_S$)}} \\
    \hline
    $\mathcal{M}_T$         & 95.78 & 88.76 & 85.90 & 90.15 \\
    $\mathcal{M}_S$         & 42.51 & 16.74 & 26.60 & 28.62 \\
    \hline
    KD                      & 61.78 & 32.45 & 54.37 & 49.53 \\
    \emph{+}Ours            & 64.49 (\textcolor{red}{+2.71}) & 35.90 (\textcolor{red}{+3.45}) & 56.79 (\textcolor{red}{+2.42}) & 52.39 (\textcolor{red}{+2.86}) \\
    SeqKD                   & 62.43 & 33.21 & 55.18 & 50.27 \\
    \emph{+}Ours            & 65.14 (\textcolor{red}{+2.71}) & 36.58 (\textcolor{red}{+3.37}) & 57.76 (\textcolor{red}{+2.58}) & 53.16 (\textcolor{red}{+2.89}) \\
    ImitKD                  & 63.12 & 31.89 & 53.92 & 49.64 \\
    \emph{+}Ours            & 65.48 (\textcolor{red}{+2.36}) & 35.32 (\textcolor{red}{+3.43}) & 57.02 (\textcolor{red}{+3.10}) & 52.61 (\textcolor{red}{+2.97}) \\
    GKD                     & 81.43 & 50.57 & 77.20 & 69.73 \\
    \emph{+}Ours            & 82.37 (\textcolor{red}{+0.94}) & 51.91 (\textcolor{red}{+1.34}) & 78.91 (\textcolor{red}{+1.71}) & 71.06 (\textcolor{red}{+1.33}) \\
    DistiLLM                & 82.95 & 51.26 & 76.68 & 70.30 \\
    \emph{+}Ours            & 84.08 (\textcolor{red}{+1.13}) & 52.44 (\textcolor{red}{+1.18}) & 77.83 (\textcolor{red}{+1.15}) & 71.45 (\textcolor{red}{+1.15}) \\
    SKD                     & 78.45 & 57.11 & 72.21 & 69.26 \\
    \emph{+}Ours            & 79.93 (\textcolor{red}{+1.48}) & 58.53 (\textcolor{red}{+1.42}) & 73.70 (\textcolor{red}{+1.49}) & 70.72 (\textcolor{red}{+1.46}) \\
    DistiLLM-2              & 85.97 & 59.53 & 78.99 & 74.83 \\
    \emph{+}Ours            & {\bf 86.96} (\textcolor{red}{+0.99}) & {\bf 60.40} (\textcolor{red}{+0.87}) & {\bf 79.82} (\textcolor{red}{+0.83}) & {\bf 75.73} (\textcolor{red}{+0.90}) \\
    \bottomrule
  \end{tabular}
  }
  % \vspace{-1.4em}
\end{table}

\subsubsection{Mathematical Reasoning in LLMs}
As shown in Table~\ref{tab:combined-results} (left), SelecTKD delivers consistent and significant improvements on mathematical reasoning benchmarks across all teacher-student setups. Specifically, with Qwen2-Math-7B-Inst as teacher and 1.5B as student, SelecTKD improves Pass@1 on GSM8K from 76.27 to 77.82 (+1.55), and on MATH from 35.58 to 36.89 (+1.31), resulting in an overall average gain from 55.93 to 57.36 (+1.43) compared to DistiLLM-2. Similar trends are observed with Qwen2.5-Math-7B-Inst, where SelecTKD raises the GSM8K Pass@1 from 81.20 to 82.45 (+1.25) and MATH from 42.94 to 43.92 (+0.98), achieving an average improvement of +1.12. Notably, this consistent gain is observed under multiple baselines (e.g., GKD, DistiLLM), highlighting that the implicit curriculum and token filtering of SelecTKD allow the student to focus on high-confidence, learnable tokens. This targeted supervision not only prevents the student from overfitting to noisy signals but also leads to improved multi-step reasoning capabilities, especially on challenging tasks like MATH.

\subsubsection{Code Generation in LLMs}
Table~\ref{tab:combined-results} (right) summarizes results on code generation benchmarks. SelecTKD achieves the highest Pass@1 accuracy on both HumanEval and MBPP for all considered model pairs. For example, using DS-Coder-6.9B-Inst and Qwen2.5-Coder-7B-Inst as teachers, SelecTKD brings HumanEval Pass@1 from 59.92 to 61.72 (+1.80) and MBPP from 75.66 to 76.27 (+0.61) for DistiLLM-2, with an overall average increase from 67.79 to 69.00 (+1.21). Similarly, on Qwen2.5-Coder-7B-Inst, HumanEval improves from 42.24 to 43.63 (+1.39), and MBPP from 62.70 to 64.32 (+1.62). These results demonstrate the superiority of selective loss masking for code tasks, where precise syntax and logical accuracy are critical. By focusing optimization on confident teacher predictions, SelecTKD facilitates more robust and stable acquisition of programming structures, surpassing all previous baselines on Pass@1.

\definecolor{mygray}{gray}{0.92}
\begin{table*}[htbp]\small
  \caption{Combined results on Math and Code benchmarks. The best {\bf Pass@1} score is highlighted in bold. `\emph{+}Ours' denotes SelecTKD.}
  \label{tab:combined-results}
  \centering
  \resizebox{0.9\linewidth}{!}{
  \begin{tabular}{lccc|ccc|ccc|ccc}
  \toprule[0.1em]
  \multirow{8}{*}{{\bf Method}} & \multicolumn{6}{c|}{{\bf Math Benchmarks}} & \multicolumn{6}{c}{{\bf Code Benchmarks}} \\
  \cmidrule(){2-13}
    & \multicolumn{3}{c|}{{\bf Qwen2-Math-7B-Inst $\rightarrow$ 1.5B}} & \multicolumn{3}{c|}{{\bf Qwen2.5-Math-7B-Inst $\rightarrow$ 1.5B}}
    & \multicolumn{3}{c|}{{\bf DS-Coder-6.9B-Inst $\rightarrow$ 1.3B}} & \multicolumn{3}{c}{{\bf Qwen2.5-Coder-7B-Inst $\rightarrow$ 1.5B}} \\
    \cmidrule(){2-13}
    & {\bf GSM8K} & {\bf MATH} & {\bf AVG.} & {\bf GSM8K} & {\bf MATH} & {\bf AVG.} & {\bf HEval}  & {\bf MBPP} & {\bf AVG.} & {\bf HEval}  & {\bf MBPP} & {\bf AVG.} \\
    & {\bf Pass@1} & {\bf Pass@1} & {\bf Pass@1} & {\bf Pass@1} & {\bf Pass@1} & {\bf Pass@1} & {\bf Pass@1} & {\bf Pass@1} & {\bf Pass@1} & {\bf Pass@1} & {\bf Pass@1} & {\bf Pass@1} \\
    \midrule
    $\mathcal{M}_{T}$    & 83.93 & 41.28 & 62.61 & 89.31 & 44.82 & 67.07 & 85.37 & 82.54 & 83.96 & 75.61 & 74.60 & 75.11 \\
    $\mathcal{M}_{S}$    & 74.53 & 25.56 & 50.05 & 77.33 & 27.14 & 52.24 & 50.61 & 72.22 & 61.42 & 30.73 & 60.84 & 45.79 \\
    \midrule[0.05em]
    {\bf GKD}                              & 75.44 & 34.16 & 54.80 & 80.21 & 40.54 & 60.38 & 54.88 & 74.34 & 64.61 & 40.85 & 61.90 & 51.38 \\
    \rowcolor{mygray} \emph{+}Ours         & 76.86 & 35.43 & 56.15 & 82.10 & 42.31 & 62.21 & 56.23 & 76.08 & 66.16 & 42.32 & 63.67 & 53.00  \\
      {\bf DistiLLM}                       & 75.59 & 34.54 & 55.07 & 81.05 & 41.14 & 61.10 & 53.65 & 74.34 & 64.00 & 39.63 & 62.17 & 50.90 \\
    \rowcolor{mygray} \emph{+}Ours         & 77.38 & 35.92 & 56.65 & 82.23 & 42.20 & 62.22 & 55.18 & 75.92 & 65.55 & 41.61 & 63.89 & 52.75 \\
      {\bf DistiLLM-2}                     & 76.27 & 35.58 & 55.93 & 81.20 & 42.94 & 62.07 & 59.92 & 75.66 & 67.79 & 42.24 & 62.70 & 52.47 \\
    \rowcolor{mygray} \emph{+}Ours         & {\bf 77.82} & {\bf 36.89} & {\bf 57.36} & {\bf 82.45} & {\bf 43.92} & {\bf 63.19} & {\bf 61.72} & {\bf 76.27} & {\bf 69.00} & {\bf 43.63} & {\bf 64.32} & {\bf 53.98} \\
    \bottomrule[0.1em]
  \end{tabular}
  }
  % \vspace{-1.4em}
\end{table*}

\subsubsection{Comprehensive Evaluation in VLMs}
To further validate SelecTKD, we extend it to vision-language models by distilling InternVL2-2B~\cite{Chen2024} and evaluating the resulting SelecTKD-VLM-2B on the Open VLM Leaderboard~\cite{2023opencompass}. As summarized in Table~\ref{tab:open_vlm}, SelecTKD-VLM-2B surpasses all existing models up to 4B parameters, outperforming larger models like Phi-3-Vision (4.2B) and TAID-VLM-2B. Notably, it achieves top or near-top accuracy on nearly all benchmarks, especially excelling on MMStar and MMBench\_V11. These results demonstrate the scalability and effectiveness of SelecTKD for compact, high-performing VLMs, highlighting its strong multimodal knowledge transfer capabilities.

\begin{table*}[t]
  \centering
  \caption{Performance of SelecTKD-VLM-2B, our new state-of-the-art VLM for models up to 4B parameters.}
  \label{tab:open_vlm}
  \resizebox{0.95\linewidth}{!}{
  \begin{tabular}{lcccccccc|c}
      \toprule
      {\bf Model} & {\bf MMBench\_V11} & {\bf MMStar} & {\bf MMMU\_VAL} & {\bf MathVista} & {\bf OCRBench} & {\bf AI2D} & {\bf HallusionBench} & {\bf MMVet} & {\bf Average} \\
      \midrule
      PaliGemma-3B-mix-448~\cite{Beyer2024}   & 65.6      & 48.3      & 34.9       & 28.7       & 61.4      & 68.3       & 32.2      & 33.1       & 46.6 \\
      MiniCPM-V-2~\cite{Yao2024}              & 65.8      & 39.1      & 38.2       & 39.8       & 60.5      & 62.9       & 36.1      & 41.0       & 47.9 \\
      Phi-3-Vision~\cite{Abdin2024}           & 65.2      & 47.7      & {\bf 46.1} & 44.6       & 63.7      & {\bf 78.4} & 39.0      & 44.1       & 53.6 \\        
      InternVL2-2B~\cite{Chen2024}            & 69.6      & 49.8      & 36.3       & 46.0       & 78.1      & 74.1       & 38.0      & 39.7       & 54.0 \\
      TAID-VLM-2B~\cite{Shing2025}            & 70.7      & 49.5      & 35.1       & 51.6       & 78.6      & 74.0       &{\bf 56.8} & 35.1       & 56.4 \\
      \rowcolor{mygray} SelecTKD-VLM-2B(Ours) &{\bf 71.3} &{\bf 50.2} & 39.8       & {\bf 52.6} &{\bf 79.5} & 77.3       & 46.5      & {\bf 45.8} & {\bf 57.9} \\
      \bottomrule
  \end{tabular} 
  } 
  % \vspace{-0.3cm}
\end{table*}

\subsection{Ablation Studies}
\subsubsection{Effect of the Verification Mechanisms}
We conduct a comprehensive ablation to evaluate the effect of different verification mechanisms within the SelecTKD framework, as summarized in Table~\ref{tab:ablation_combined}. Compared to the non-selective Vanilla KD baseline, all selective schemes yield superior average win rates, demonstrating the crucial role of targeted loss application. Among the studied approaches, token-based selection strategies provide more robust improvements over the Hellinger-based criterion. Notably, the Non-Greedy Spec-$k$ strategy consistently achieves the highest win rates, suggesting that allowing for flexible acceptance rather than strict selection can further enhance student learning by embracing a broader range of valid learning signals.

\subsubsection{Impact of Hyperparameters \texorpdfstring{$k$}{k} and \texorpdfstring{$\beta$}{beta}}
We further analyze the influence of the key hyperparameters $k$ (the number of candidates) and $\beta$ (the weighting for rejected tokens). As reported in Table~\ref{tab:ablation_combined}, performance is not highly sensitive within a reasonable range of $k$, but peaks at $k = 5$. Adjusting $\beta$ reveals that while hard gating ($\beta=0$) is already effective, introducing a small non-zero value ($\beta=0.01$) consistently improves results. This indicates that weakly training on rejected tokens serves as a beneficial regularizer, likely providing implicit exposure to valuable negative information. Conversely, increasing $\beta$ further erodes the advantages of selective learning and causes performance to approach that of Vanilla KD. Overall, the best results are obtained with $k=5$ and $\beta=0.01$.

\begin{table}[htbp]\small
  \caption{
    Comprehensive ablation on verification mechanisms, selection types, and key hyperparameters $k$ and $\beta$ for SelecTKD. Results are reported as average win rate (\%) on instruction-following benchmarks (Qwen2-7B-Inst $\rightarrow$ Qwen2-1.5B).
  }
  \label{tab:ablation_combined}
  \centering
  \resizebox{1.0\linewidth}{!}{
  \begin{tabular}{lccccl}
    \toprule
    {\bf Method}              & {\bf Verification}   & {$k$}    & {$\beta$}  & {\bf Avg. Win Rate (\%)} \\
    \midrule
    Vanilla KD                & None                 & --       & 1          & 41.19   \\
    SelecTKD (H)              & Hellinger-based      & --       & --         & 41.70   \\
    SelecTKD (G)              & Greedy Top-$k$       & 3        & 0.01       & 42.98   \\
    SelecTKD (G)              & Greedy Top-$k$       & 5        & 0.01       & 43.09   \\
    SelecTKD (G)              & Greedy Top-$k$       & 7        & 0.01       & 43.04   \\
    SelecTKD (G)              & Greedy Top-$k$       & 5        & 0          & 42.27   \\
    \rowcolor{mygray}
    SelecTKD (NG)             & Non-Greedy Spec-$k$  & 3        & 0.01       & 43.22   \\
    \rowcolor{mygray}
    SelecTKD (NG)             & Non-Greedy Spec-$k$  & 5        & 0.01       & {\bf 43.33} \\
    \rowcolor{mygray}
    SelecTKD (NG)             & Non-Greedy Spec-$k$  & 7        & 0.01       & 43.26   \\
    \rowcolor{mygray}
    SelecTKD (NG)             & Non-Greedy Spec-$k$  & 10       & 0.01       & 43.13   \\
    \rowcolor{mygray}
    SelecTKD (NG)             & Non-Greedy Spec-$k$  & 5        & 0.1        & 43.20   \\
    \rowcolor{mygray}
    SelecTKD (NG)             & Non-Greedy Spec-$k$  & 5        & 0.05       & 43.18   \\
    \rowcolor{mygray}
    SelecTKD (NG)             & Non-Greedy Spec-$k$  & 5        & 0          & 42.83   \\
    \bottomrule
  \end{tabular}
  }
  \vspace{-0.3cm}
\end{table}

\subsubsection{On the Role of Loss Function Geometry}
To investigate whether loss function geometry is the dominant factor in distillation, we experiment with various commonly used distillation losses within the SelecTKD framework. All models are trained on the same mixed dataset generated offline by both teacher and student (see Figure~\ref{fig:SpecKD_performance_analysis} (a)). We observe that while different losses may converge at different rates, their final performance is similar when training is sufficiently long. This supports our hypothesis that, with selective loss application, the choice of loss function is less critical than the control over where and when the learning signal is applied. Losses incorporating Skew KL consistently outperform others, in line with prior findings.

\begin{figure}[ht]
  \centering
  \begin{subfigure}[b]{0.233\textwidth}
    \centering
    \includegraphics[width=\textwidth]{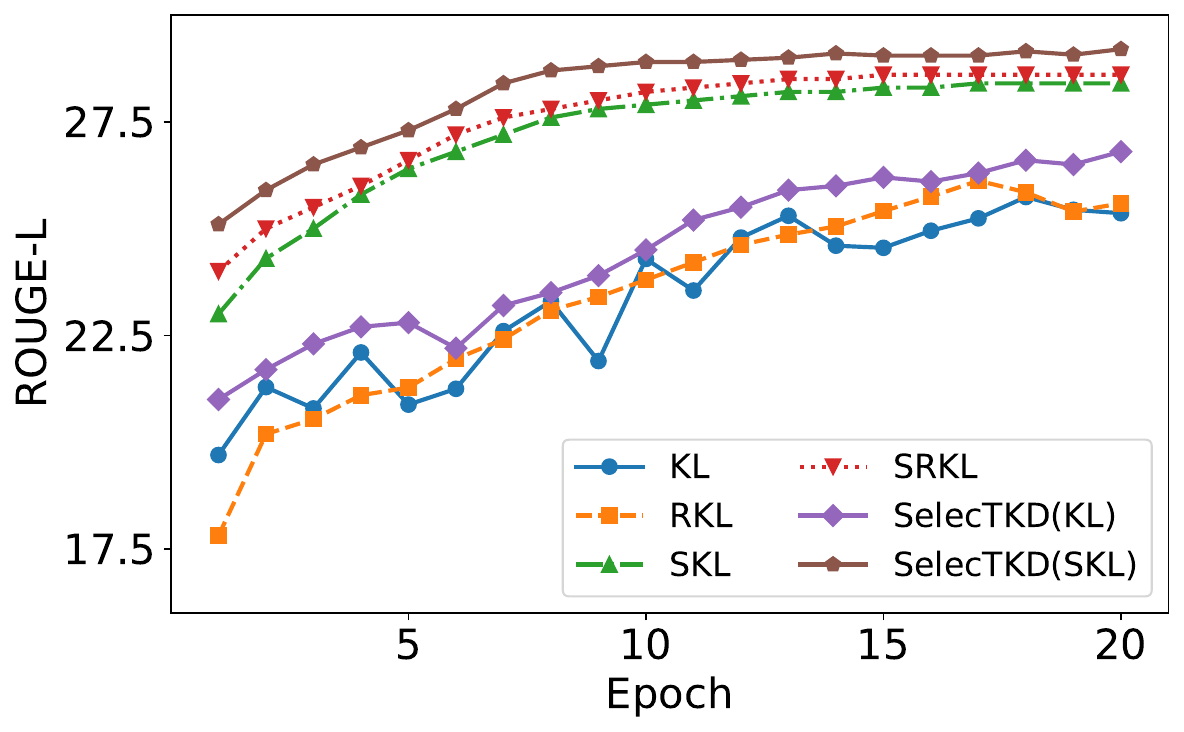}
    \caption{}
  \end{subfigure}
  \hfill
  \begin{subfigure}[b]{0.233\textwidth}
    \centering
    \includegraphics[width=\textwidth]{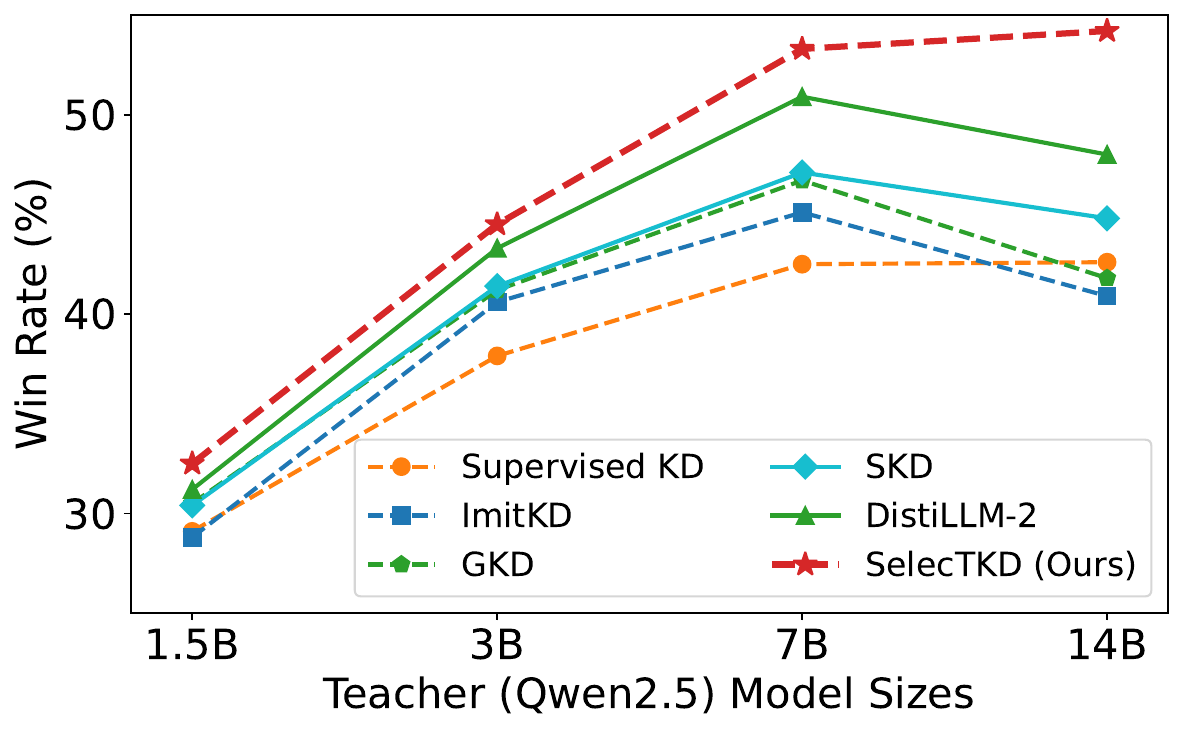}
    \caption{}
  \end{subfigure}
  \caption{Performance analysis of SelecTKD in Qwen2-1.5B: (a) Comparison of different distillation methods during training; (b) Robustness of SelecTKD to increasingly larger teachers.}
  \label{fig:SpecKD_performance_analysis}
  \vspace{-0.3cm}
\end{figure}

\subsubsection{Mitigating the Curse of the Powerful Teacher}
A well-known challenge in knowledge distillation is that using a much larger teacher can sometimes degrade student performance, as noted by~\cite{Zhong2024}. We hypothesize this is due to the higher-entropy, more complex distributions produced by large teachers, which can overwhelm a capacity-limited student. SelecTKD addresses this by filtering out such noisy signals. To validate this, we fix the student (Qwen2-1.5B) and distill from teachers of increasing size. As shown in Figure~\ref{fig:SpecKD_performance_analysis} (b), Vanilla KD performance plateaus and then drops as teacher size increases, confirming the ``curse of the powerful teacher.'' In contrast, SelecTKD enables the student to benefit monotonically from stronger teachers, highlighting the importance of selective loss for robust knowledge transfer from frontier models.

\subsection{Discussion of SelecTKD Optimization}
\subsubsection{Connection to Implicit Curriculum Learning}
SelecTKD naturally induces an implicit, adaptive curriculum~\cite{liu2024let}. As training progresses, the Token Acceptance Rate (TAR) increases quasi-monotonically, see Figure~\ref{fig:SpecKD_optimal_analysis} (a), indicating that the student is mastering ``easy'' tokens first. As the model improves, more challenging tokens are gradually included in the loss, automatically adjusting the learning difficulty to the student's current capability. This self-paced curriculum eliminates the need for manual scheduling and contributes to stable and efficient training.

\begin{figure}[ht]
  \centering
  \begin{subfigure}[b]{0.233\textwidth}
    \centering
    \includegraphics[width=\textwidth]{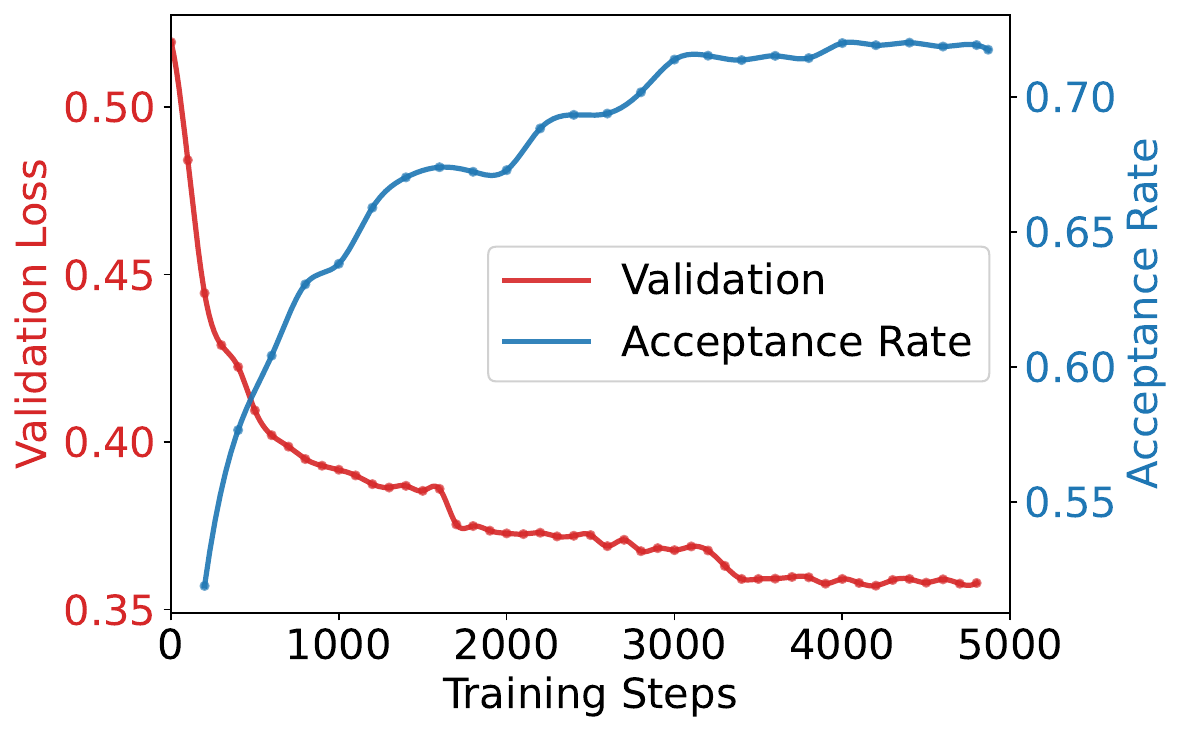}
    \caption{}
  \end{subfigure}
  \hfill
  \begin{subfigure}[b]{0.203\textwidth}
      \centering
      \includegraphics[width=\textwidth]{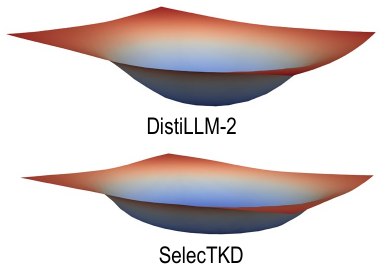}
      \caption{}
  \end{subfigure}
  \caption{Analysis of SelecTKD dynamics: (a) validation loss generally decreases as the Token Acceptance Rate (TAR) increases, indicating an implicit curriculum; (b) SelecTKD yields a flatter loss landscape~\cite{li2018visualizing} than strong baselines (e.g., DistiLLM-2), which correlates with improved generalization.}
  \label{fig:SpecKD_optimal_analysis}
  \vspace{-0.3cm}
\end{figure}

\subsubsection{Impact on Loss Landscape and Generalization}
By masking loss contributions from tokens with high teacher-student divergence, SelecTKD smooths the loss landscape and mitigates the effect of sharp, noisy gradients. As shown in Figure~\ref{fig:SpecKD_optimal_analysis} (b), SelecTKD yields a flatter loss surface compared to strong baselines such as DistiLLM-2. A smoother loss landscape is empirically linked to better generalization~\cite{izmailov2018averaging}, which aligns with the robust performance improvements observed in our experiments.

\subsubsection{Inference Speedup of Speculative Decoding}
We further assess SelecTKD by measuring its impact on speculative decoding, where effective alignment between drafter and verifier models is essential for fast and efficient generation~\citep{Zhou2023}. Table~\ref{tab:speculative} summarizes results using Llama-68m~\citep{miao2024specinfer} as drafter and Phi-3-medium or Phi-3.5-mini~\citep{abdin2024phi} as verifiers. All draft models share the same training data but are distilled via SFT, GKD, DistiLLM, or SelecTKD. We report both average speedup and token acceptance rate as key metrics. Across both verifier settings, SelecTKD consistently achieves the best speedup and acceptance rates---for example, with Phi-3-medium, SelecTKD yields $\times$2.05 speedup and 0.492 acceptance, outperforming other strong baselines, and similar gains are seen with Phi-3.5-mini. These improvements demonstrate that SelecTKD's selective token alignment produces better-matched draft-verifier distributions, directly leading to higher speculative decoding efficiency without compromising acceptance. Notably, the higher acceptance rates attained with SelecTKD indicate improved token-level agreement, implying a narrower drafter-verifier distribution gap compared to prior methods~\citep{Zhou2023}. Thus, targeted token selection and weighting in distillation enable both theoretical and practical advances for inference with modern language models.

\begin{table}[ht]
  \centering
  \caption{Inference speedup (Spd.) and token acceptance rate (Acpt.) in the speculative decoding setting. Results compare different draft models trained via various KD methods for both Phi-3-medium and Phi-3.5-mini verifiers.}
  \label{tab:speculative}
    \resizebox{1.0\linewidth}{!}{
    \begin{tabular}{l|c|cccccc}
      \toprule[0.1em]
      Phi- & & {\bf SFT} & {\bf GKD} & {\bf DistiLLM} & {\bf DistiLLM-2} & {\bf SelecTKD} \\ \midrule
      \multirow{2}{*}{3-medium} & Spd.\,($\uparrow$) & $\times$1.32 & $\times$1.64 & $\times$1.71 & $\times$1.97 & {\bf $\times$2.05} \\
      & Acpt.\,($\uparrow$)                          & 0.412        & 0.464        & 0.469        & 0.487        & {\bf 0.492} \\
      \midrule
      \multirow{2}{*}{3.5-mini} & Spd.\,($\uparrow$) & $\times$1.24 & $\times$1.58 & $\times$1.65 & $\times$1.84 & {\bf $\times$1.91} \\
      & Acpt.\,($\uparrow$)                          & 0.397        & 0.443        & 0.452        & 0.522        & {\bf 0.534} \\
      \bottomrule[0.1em]
    \end{tabular}}
    \vspace{-0.3cm}
\end{table}

\section{Limitations and Future Work}
Firstly, our experiments evaluate SelecTKD with LLM teachers of up to 9B parameters and an 8B VLM teacher, scaling up to frontier-scale models remains an open challenge. Additionally, we currently use a fixed Top-$k$ and a global rejected-token weight $\beta$; adaptively learning $k$/$\beta$, or conditioning them on token- or context-level uncertainty, may further enhance stability and efficiency. While our study covers single-turn, text-centric tasks and a vision-language scenario, extending selective token-weighted distillation to multi-image/video, speech, or preference-aligned pipelines is promising. Overall, SelecTKD reframes distillation from loss engineering to selective supervision, and we expect this principle to underpin more robust and efficient compact models.

\section{Conclusion}
We identify a central limitation of conventional LLM distillation: uniform, indiscriminate token-wise supervision that forces students to imitate high-entropy or noisy teacher predictions, which is especially detrimental when there are large capacity gaps. We introduce SelecTKD, a plug-and-play framework using a simple Top-$k$ propose-and-verify token mechanism—via greedy and non-greedy Spec-$k$ variants—to determine where to apply learning. Accepted tokens receive full loss, while rejected tokens are masked or weakly weighted. SelecTKD is objective-agnostic (KL, RKL, SKL, SRKL) and compatible with on- and off-policy data. Across instruction following, mathematical reasoning, code generation, and a VLM task, SelecTKD consistently improves strong baselines without architectural changes or extra reference models. Analyses of Token Acceptance Rate (TAR) and the resulting smoother loss landscape corroborate an implicit curriculum and stable optimization.

{
    \small
    \bibliographystyle{ieeenat_fullname}
    \bibliography{main}
}

\clearpage
\setcounter{page}{1}
\maketitlesupplementary

% Use lettered section numbering (A, B, C) in the supplementary
\setcounter{section}{0}
\renewcommand{\thesection}{\Alph{section}}
\renewcommand{\thesubsection}{\Alph{section}.\arabic{subsection}}
\renewcommand{\thesubsubsection}{\Alph{section}.\arabic{subsection}.\arabic{subsubsection}}

\section{Additional Results}
\label{sec:additional_results}
As discussed in Section~\ref{sec:intro}, our preliminary experiments suggest that different loss function geometries (e.g., forward/reverse/skewed KL) converge to similar fixed points despite differing optimization dynamics. To further validate this claim and demonstrate its robustness across different evaluation protocols, we conduct comprehensive experiments on multiple evaluation platforms, including DeepSeek, Qwen, OpenAI, and Kimi. These platforms employ distinct judging criteria and model preferences, providing a diverse testbed for assessing the generality of our hypothesis.

\begin{figure}[ht]
  \centering
  \includegraphics[width=1.0\linewidth]{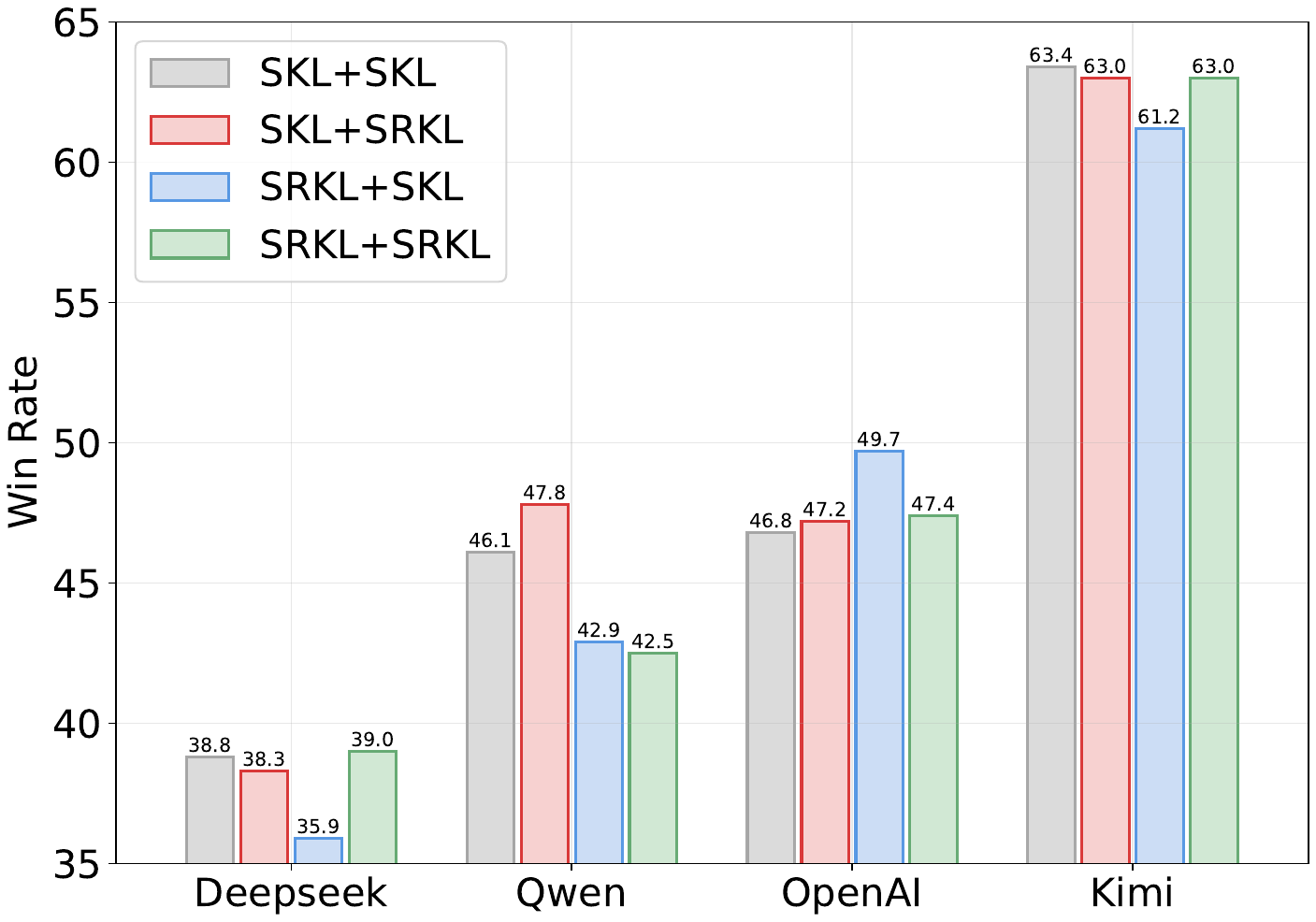}
  \caption{
    Performance comparison of symmetric and asymmetric loss function combinations on Qwen2-7B-Inst $\rightarrow$ Qwen2-1.5B-SFT across four different evaluation platforms. The chart shows the Win Rate of four loss function combinations: two symmetric forms (SKL+SKL, SRKL+SRKL) and two asymmetric forms (SKL+SRKL, SRKL+SKL). Results are evaluated using DeepSeek (deepseek-chat), Qwen (qwen-plus-2025-01-25), OpenAI (gpt-4o), and Kimi (moonshot-v1-8k) as judges. The consistent performance patterns across platforms indicate that loss function geometry is not the dominant factor in determining final student performance, supporting our reframing of the distillation problem from ``how to measure divergence'' to ``where to apply supervision.''
  }
  \label{diff_eval_platform_performance_comparison}
\end{figure}

As illustrated in Figure~\ref{diff_eval_platform_performance_comparison}, the performance differences between symmetric and asymmetric loss combinations remain relatively small across all four evaluation platforms, with variations typically within 1--2\% in win rate. This consistency across diverse judging criteria strengthens our argument that the choice of loss function geometry, while affecting training dynamics, does not fundamentally alter the achievable performance ceiling. Instead, the key insight lies in selectively applying supervision to high-confidence teacher signals, which is the core principle of SelecTKD.

Table~\ref{tab:appendix_ablation_combined} presents the complete ablation study on verification mechanisms and hyperparameters, extending the analysis in Section~\ref{sec:experiments}. The results demonstrate that while different verification strategies (Hellinger-based, Greedy Top-$k$, Non-Greedy Spec-$k$) yield varying degrees of improvement over Vanilla KD, the Non-Greedy Spec-$k$ variant consistently achieves the best performance. Furthermore, the robustness of SelecTKD to hyperparameter choices (particularly $k \in \{3,5,7,10\}$ and $\beta \in \{0,0.01,0.05,0.1\}$) indicates that the framework is practical and easy to deploy without extensive hyperparameter tuning.

\begin{table}[htbp]\small
  \caption{
    Comprehensive ablation on verification mechanisms, selection types, and key hyperparameters $k$ and $\beta$ for SelecTKD. Results are reported as average win rate (\%) on instruction-following benchmarks (Qwen2-7B-Inst $\rightarrow$ Qwen2-1.5B).
  }
  \label{tab:appendix_ablation_combined}
  \centering
  \resizebox{1.0\linewidth}{!}{
  \begin{tabular}{lccccl}
    \toprule
    {\bf Method}              & {\bf Verification}   & {$k$}    & {$\beta$}  & {\bf Avg. Win Rate (\%)} \\
    \midrule
    Vanilla KD                & None                 & --       & 1          & 41.19   \\
    SelecTKD (H)              & Hellinger-based      & --       & --         & 41.70   \\
    SelecTKD (G)              & Greedy Top-$k$       & 3        & 0.01       & 42.98   \\
    SelecTKD (G)              & Greedy Top-$k$       & 5        & 0.01       & 43.09   \\
    SelecTKD (G)              & Greedy Top-$k$       & 7        & 0.01       & 43.04   \\
    SelecTKD (G)              & Greedy Top-$k$       & 10       & 0.01       & 42.95   \\
    SelecTKD (G)              & Greedy Top-$k$       & 5        & 0.1        & 42.82   \\
    SelecTKD (G)              & Greedy Top-$k$       & 5        & 0.05       & 42.86   \\
    SelecTKD (G)              & Greedy Top-$k$       & 5        & 0          & 42.27   \\
    \rowcolor{mygray}
    SelecTKD (NG)             & Non-Greedy Spec-$k$  & 3        & 0.01       & 43.22   \\
    \rowcolor{mygray}
    SelecTKD (NG)             & Non-Greedy Spec-$k$  & 5        & 0.01       & {\bf 43.33} \\
    \rowcolor{mygray}
    SelecTKD (NG)             & Non-Greedy Spec-$k$  & 7        & 0.01       & 43.26   \\
    \rowcolor{mygray}
    SelecTKD (NG)             & Non-Greedy Spec-$k$  & 10       & 0.01       & 43.13   \\
    \rowcolor{mygray}
    SelecTKD (NG)             & Non-Greedy Spec-$k$  & 5        & 0.1        & 43.20   \\
    \rowcolor{mygray}
    SelecTKD (NG)             & Non-Greedy Spec-$k$  & 5        & 0.05       & 43.18   \\
    \rowcolor{mygray}
    SelecTKD (NG)             & Non-Greedy Spec-$k$  & 5        & 0          & 42.83   \\
    \bottomrule
  \end{tabular}
  }
\end{table}

\section{Theoretical Analysis of SelecTKD}
\label{sec:theoretical_analysis}
In this section, we provide a rigorous theoretical analysis of SelecTKD, focusing on the monotonic improvement property of the Token Acceptance Rate (TAR) during training. This analysis establishes a mathematical foundation for the implicit curriculum learning effect observed empirically in our experiments (Section~\ref{sec:experiments}). The following theorem provides a lower bound on the TAR increment per update step, which guarantees that the student-teacher alignment improves monotonically under appropriate conditions.

\begin{theorem}
\label{thm:tar_lower_bound}
Assume a sufficiently small learning rate $\eta > 0$ and that the language model satisfies standard Lipschitz continuity conditions~\cite{cisse2017parseval}. Then, for both SelecTKD variants (Greedy Top-$k$ and Non-Greedy Spec-$k$), each gradient update step improves the Token Acceptance Rate (TAR), with the improvement lower-bounded as:
{\small
\begin{equation}
\label{eq:tar_lower_bound}
\mathrm{TAR}_{t+1} - \mathrm{TAR}_t \geq \eta\,\kappa\,(1-\mathrm{TAR}_t),
\end{equation}
}
where $\kappa > 0$ is a positive constant that reflects the average ease of correcting a rejected proposal. The constant $\kappa$ depends on: (i) the teacher's confidence margin at the Top-$k$ decision boundary, (ii) the model's smoothness properties (Lipschitz constant), and (iii) the gradient magnitude for tokens near the acceptance threshold.
\end{theorem}

\begin{proof}
Let $c$ denote a context tuple $(x, y_{<t})$ representing the input and previous tokens. Define the acceptance indicator function at training step $t$ as $A_t(c) = \mathbb{I}\!\big(V_t(c)=1\big)$, where $V_t$ is the SelecTKD verification weight. For the Greedy Top-$k$ variant, $V_t(c)=1$ if and only if $\arg\max_{y} q_\theta(y|c) \in \mathrm{Top}_k(p(\cdot|c))$; for the Non-Greedy Spec-$k$ variant, $V_t(c)=1$ if and only if at least one of the $k$ sampled candidates from $q_\theta(\cdot|c)$ passes the speculative acceptance test. The Token Acceptance Rate at step $t$ is defined as $\mathrm{TAR}_t = \mathbb{E}_c[A_t(c)]$, where the expectation is taken over the data distribution.

The change in TAR after one gradient update step is:
{\small
\begin{equation}
\label{eq:tar_change}
\mathrm{TAR}_{t+1} - \mathrm{TAR}_t = \mathbb{E}_c[A_{t+1}(c) - A_t(c)].
\end{equation}
}
Under the assumption of a sufficiently small learning rate $\eta$ and smooth model dynamics, the probability of an accepted token becoming rejected is negligible compared to the reverse transition. This follows from the fact that gradient updates are local and smooth, and accepted tokens are already within the teacher's high-confidence region. Therefore, the TAR increment is dominated by tokens that transition from rejected to accepted:
{\small
\begin{multline}
\label{eq:tar_dominated}
\mathrm{TAR}_{t+1} - \mathrm{TAR}_t 
\approx \Pr(A_{t+1}=1,\,A_t=0) 
\\
- \underbrace{\Pr(A_{t+1}=0,\,A_t=1)}_{\approx 0} 
\geq \Pr(A_{t+1}=1,\,A_t=0).
\end{multline}
}
Let $B_t = \{c: A_t(c) = 0\}$ denote the set of rejected contexts at step $t$. The probability of a context being rejected is $\Pr(c \in B_t) = 1 - \mathrm{TAR}_t$. The improvement in TAR comes from tokens in $B_t$ that become accepted after the update:
{\small
\begin{multline}
\label{eq:tar_improvement}
\Pr(A_{t+1}=1,\,A_t=0) = \Pr(c \in B_t) \cdot \Pr(A_{t+1}=1 \mid c \in B_t) \\
= (1-\mathrm{TAR}_t) \cdot \Pr(A_{t+1}=1 \mid c \in B_t).
\end{multline}
}
The conditional probability $\Pr(A_{t+1}=1 \mid c \in B_t)$ represents the likelihood that a previously rejected context becomes accepted after one update. This probability can be lower-bounded by $\eta \kappa$, where $\kappa$ depends on: (i) the gradient magnitude for tokens near the Top-$k$ boundary, (ii) the teacher's probability margin between the $k$-th and $(k+1)$-th ranked tokens, and (iii) the model's Lipschitz smoothness, which ensures that small parameter updates lead to predictable distribution changes. Formally, $\kappa$ can be expressed as:
{\small
\begin{equation}
\kappa = \min_{c \in B_t, y \in \mathrm{Boundary}_k(c)} \left\{ \frac{|\nabla_\theta \log q_\theta(y|c)| \cdot \mathrm{margin}_k(p(\cdot|c))}{L} \right\},
\end{equation}
}
where $\mathrm{Boundary}_k(c)$ denotes tokens near the Top-$k$ decision boundary, $\mathrm{margin}_k(p(\cdot|c))$ is the probability gap between the $k$-th and $(k+1)$-th ranked tokens under the teacher distribution, and $L$ is the Lipschitz constant. Combining equations~\eqref{eq:tar_dominated} and~\eqref{eq:tar_improvement}, we obtain:
{\small
\begin{equation}
\mathrm{TAR}_{t+1} - \mathrm{TAR}_t \geq (1-\mathrm{TAR}_t) \cdot \eta \kappa,
\end{equation}}
which completes the proof.
\end{proof}

Theorem~\ref{thm:tar_lower_bound} provides a unified mathematical foundation for the implicit and adaptive curriculum induced by SelecTKD. The lower bound in equation~\eqref{eq:tar_lower_bound} reveals several important properties:

\begin{enumerate}
\item \textbf{Monotonic improvement:} The theorem guarantees that TAR increases (or at least does not decrease) at each step, ensuring stable training dynamics.
\item \textbf{Adaptive curriculum:} The term $(1-\mathrm{TAR}_t)$ quantifies the proportion of ``unlearned'' or misaligned tokens. The bound shows that the largest improvements occur early in training when misalignment is high ($\mathrm{TAR}_t \approx 0$), while the improvement rate naturally decreases as the student approaches the teacher ($\mathrm{TAR}_t \to 1$). This creates an automatic, self-paced curriculum that starts with easier tokens and gradually incorporates more challenging ones.
\item \textbf{Robustness to hyperparameters:} The constant $\kappa$ depends on the teacher's confidence margin, which is typically well-separated for high-quality teachers. This ensures that the improvement bound remains meaningful across different model sizes and tasks.
\end{enumerate}

This theoretical result is empirically supported by the TAR curves observed in our experiments (see Figure~\ref{fig:SpecKD_optimal_analysis} (a) in the main text), which exhibit the predicted quasi-monotonic increase and saturating behavior. The theorem confirms that SelecTKD prioritizes correcting the student's most significant deviations first, leading to stable optimization and improved generalization, as evidenced by the flatter loss landscapes discussed in Section~\ref{sec:experiments}.

\section{Gradient Derivations}
\label{sec:appendix_gradient}
This section provides detailed derivations of the gradients for various divergence measures used in knowledge distillation. These derivations support the convergence analysis presented in the main text and clarify the relationship between different loss functions. We begin with the forward and reverse KL divergences, then extend to the skewed variants (SKL and SRKL) used in modern distillation frameworks.

\subsection{Derivation of FKL Gradient}
\label{app:fkl_grad}

The forward KL divergence measures the expected log-likelihood ratio under the teacher distribution, encouraging the student to cover all modes of the teacher. Consider the forward KL divergence term at time step $t$ and token $v_i$:
{\small
\begin{equation}
\label{eq:app_fkl_def}
D_{\text{FKL}}^{(t,i)}(p, q_\theta) = p_i \, \log \frac{p_i}{q_i},
\end{equation}}
where we use the shorthand notation:
{\small
\begin{equation}
\label{eq:app_fkl_defs}
p_i := p(v_i \mid \boldsymbol{y}_{<t}, \boldsymbol{x}), \quad q_i := q_\theta(v_i \mid \boldsymbol{y}_{<t}, \boldsymbol{x}).
\end{equation}}
Since the teacher distribution $p_i$ is fixed (treated as a constant with respect to the student parameters $\theta$), the derivative with respect to $q_i$ is:
{\small
\begin{equation}
\label{eq:app_fkl_grad_q}
\frac{\partial}{\partial q_i} D_{\text{FKL}}^{(t,i)}(p, q_\theta) = -\, \frac{p_i}{q_i}.
\end{equation}}
The negative sign indicates that increasing $q_i$ reduces the divergence when $p_i > 0$, which aligns with the mass-covering behavior of forward KL. Using the chain rule through the softmax function, the gradient with respect to the student distribution is:
{\small
\begin{equation}
\label{eq:app_fkl_grad_full}
\frac{\partial}{\partial q_\theta(v_i \mid \boldsymbol{y}_{<t}, \boldsymbol{x})} D_{\text{FKL}}^{(t,i)}(p, q_\theta)
= -\frac{p(v_i \mid \boldsymbol{y}_{<t}, \boldsymbol{x})}{q_\theta(v_i \mid \boldsymbol{y}_{<t}, \boldsymbol{x})}.
\end{equation}}
This gradient is large when $p_i$ is high and $q_i$ is low, emphasizing the importance of learning high-probability teacher tokens.

\subsection{Derivation of RKL Gradient}
\label{app:rkl_grad}

The reverse KL divergence measures the expected log-likelihood ratio under the student distribution, encouraging mode-seeking behavior. For the reverse KL divergence at position $(t,i)$:
{\small
\begin{equation}
\label{eq:app_rkl_def}
D_{\text{RKL}}^{(t,i)}(p, q_\theta) = q_i \, \log \frac{q_i}{p_i},
\end{equation}}
with the same probability definitions:
{\small
\begin{equation}
\label{eq:app_rkl_defs}
p_i := p(v_i \mid \boldsymbol{y}_{<t}, \boldsymbol{x}), \quad q_i := q_\theta(v_i \mid \boldsymbol{y}_{<t}, \boldsymbol{x}).
\end{equation}}
Differentiating with respect to $q_i$ using the product rule:
{\small
\begin{align}
\label{eq:app_rkl_grad_q}
\frac{\partial}{\partial q_i} D_{\text{RKL}}^{(t,i)}(p, q_\theta)
&= \frac{\partial}{\partial q_i} \left[ q_i \log q_i - q_i \log p_i \right] \\
&= \left( \log q_i + 1 \right) - \log p_i \\
&= \log \frac{q_i}{p_i} + 1.
\end{align}}
The gradient is positive when $q_i > \frac{p_i}{e}$ and negative when $q_i < \frac{p_i}{e}$, pushing the student distribution toward the teacher. The full gradient with respect to the student parameters is:
{\small
\begin{equation}
\label{eq:app_rkl_grad_full}
\frac{\partial}{\partial q_\theta(v_i \mid \boldsymbol{y}_{<t}, \boldsymbol{x})} D_{\text{RKL}}^{(t,i)}(p, q_\theta)
= \log \frac{q_\theta(v_i \mid \boldsymbol{y}_{<t}, \boldsymbol{x})}{p(v_i \mid \boldsymbol{y}_{<t}, \boldsymbol{x})} + 1.
\end{equation}}
Note that the gradient depends on the ratio $q_i/p_i$, making it sensitive to tokens where the student over- or under-estimates the teacher's probability.

\subsection{Convergence Analysis of FKL and RKL}
\label{FRKL_proof}
We now analyze the convergence properties of forward and reverse KL divergences, showing that both objectives converge to the same fixed point $q_\theta = p$ under appropriate conditions. This analysis supports the claim in the main text that different loss geometries share the same optimal solution, despite differing optimization dynamics.

We follow the notation established in the main text: $p$ denotes the teacher distribution, $q_\theta$ the student distribution parameterized by $\theta$, and $z^q_j$ the student logit for token $j$ in the vocabulary $\mathcal{V} = \{1, \ldots, |\mathcal{V}|\}$.

\begin{proof}[Convergence of Forward KL]
The convergence condition for forward KL requires that all gradients with respect to student logits vanish:
{\small
\begin{equation}
\label{eq:fkl_convergence_condition}
\frac{\partial D_{\text{FKL}}(p, q_\theta)}{\partial z^q_j} = 0, \quad \forall j \in \{1, \ldots, |\mathcal{V}|\}.
\end{equation}}
Using the standard result for cross-entropy with soft targets, the gradient of forward KL with respect to student logits is:
{\small
\begin{equation}
\label{eq:fkl_grad_logits}
\frac{\partial D_{\text{FKL}}(p, q_\theta)}{\partial z^q_j} = q_\theta(j|\mathbf{y}_{<t}, \boldsymbol{x}) - p(j|\mathbf{y}_{<t}, \boldsymbol{x}).
\end{equation}}
This follows from the fact that forward KL is equivalent to cross-entropy with the teacher distribution as targets. The gradient in equation~\eqref{eq:fkl_grad_logits} vanishes if and only if:
{\small
\begin{equation}
\label{eq:fkl_fixed_point}
q_\theta(j|\mathbf{y}_{<t}, \boldsymbol{x}) = p(j|\mathbf{y}_{<t}, \boldsymbol{x}), \quad \forall j \in \{1, \ldots, |\mathcal{V}|\}.
\end{equation}}
Since forward KL is convex in $q_\theta$ and the constraint $\sum_j q_\theta(j) = 1$ defines a convex set, the stationary point in equation~\eqref{eq:fkl_fixed_point} is the unique global minimum. Therefore, forward KL converges to $q_\theta = p$.
\end{proof}

\begin{proof}[Convergence of Reverse KL]
The convergence condition for reverse KL is:
{\small
\begin{equation}
\label{eq:rkl_convergence_condition}
\frac{\partial D_{\text{RKL}}(p, q_\theta)}{\partial z^q_j} = 0, \quad \forall j \in \{1, \ldots, |\mathcal{V}|\}.
\end{equation}}
Let $q_\theta = \mathrm{softmax}(z^q)$ denote the student distribution obtained by applying the softmax function to logits $z^q$. The gradient of reverse KL with respect to logits can be computed using the chain rule. A compact matrix form is:
{\small
\begin{equation}
\label{eq:rkl_grad_matrix}
\frac{\partial D_{\text{RKL}}(p, q_\theta)}{\partial z^q} 
= J_{\mathrm{softmax}}(z^q)\, \big( \log q_\theta - \log p + \mathbf{1} \big),
\end{equation}}
where $J_{\mathrm{softmax}}(z^q) = \mathrm{diag}(q_\theta) - q_\theta q_\theta^\top$ is the softmax Jacobian matrix, $[\log q_\theta]_i = \log q_\theta(i)$, $[\log p]_i = \log p(i)$, and $\mathbf{1}$ is the all-ones vector of dimension $|\mathcal{V}|$.

When $q_\theta = p$ (elementwise equality), the term $(\log q_\theta - \log p + \mathbf{1})$ reduces to $\mathbf{1}$. Since the softmax Jacobian satisfies $J_{\mathrm{softmax}}(z^q) \mathbf{1} = \mathbf{0}$ (due to the constraint $\sum_j q_\theta(j) = 1$), the gradient in equation~\eqref{eq:rkl_grad_matrix} is zero. Therefore, $q_\theta = p$ is a stationary point.

To establish uniqueness, note that reverse KL is strictly convex in $q_\theta$ on the interior of the probability simplex. This follows from the fact that the function $f(q) = q \log(q/p)$ is strictly convex for $q > 0$ and $p > 0$. Therefore, the stationary point $q_\theta = p$ is the unique global minimum, and reverse KL converges to this fixed point.
\end{proof}

The convergence analysis above demonstrates that both forward and reverse KL divergences converge to the same fixed point $q_\theta = p$, despite their different optimization dynamics (mass-covering vs. mode-seeking). This result supports the empirical observation in the main text that different loss geometries yield similar final performance when training is sufficiently long.

\subsection{Convergence Analysis of SKL and SRKL}
\label{SKL_SRKL_proof}
We now extend the convergence analysis to the skewed variants (SKL and SRKL) introduced in DistiLLM~\cite{Ko2024} and used in DistiLLM-2~\cite{Ko2025}. These variants interpolate between teacher and student distributions, providing a flexible mechanism to balance mass-covering and mode-seeking behaviors.

\begin{proof}[Convergence of Skew KL (SKL)]
The convergence condition for SKL requires:
{\small
\begin{equation}
\label{eq:skl_convergence_condition}
\frac{\partial \mathcal{L}_{\mathrm{SKL}}}{\partial z_j} = 0, \quad \forall j \in \{1, \ldots, |\mathcal{V}|\}.
\end{equation}}
Recall from Section~\ref{subsec:preliminaries} that SKL is defined as:
{\small
\begin{equation}
\mathcal{L}_{\mathrm{SKL}} = D_{\mathrm{KL}}(p \| m), \quad \text{where } m = \alpha p + (1-\alpha) q_\theta,
\end{equation}}
and $\alpha \in [0,1)$ is the skew coefficient. Since $D_{\mathrm{KL}}(p \| \cdot)$ is strictly convex in its second argument (the mixed distribution $m$), and $m$ is an affine function of $q_\theta$, the unique minimum of SKL is achieved when $m = p$. Substituting the definition of $m$:
{\small
\begin{equation}
\alpha p + (1-\alpha) q_\theta = p \quad \Longleftrightarrow \quad (1-\alpha)(q_\theta - p) = 0.
\end{equation}}
For $\alpha < 1$, this implies $q_\theta = p$. When $\alpha = 1$, SKL degenerates to a constant (zero) and the condition is trivially satisfied. Therefore, $q_\theta = p$ is the unique stationary point (and global minimum) for SKL, establishing convergence to the teacher distribution.
\end{proof}

\begin{proof}[Convergence of Skew Reverse KL (SRKL)]
The convergence condition for SRKL is:
{\small
\begin{equation}
\label{eq:srkl_convergence_condition}
\frac{\partial \mathcal{L}_{\mathrm{SRKL}}}{\partial z_j} = 0, \quad \forall j \in \{1, \ldots, |\mathcal{V}|\}.
\end{equation}}
Recall that SRKL is defined as:
{\small
\begin{equation}
\mathcal{L}_{\mathrm{SRKL}} = D_{\mathrm{KL}}(q_\theta \| m'), \quad \text{where } m' = (1-\alpha) p + \alpha q_\theta.
\end{equation}}
Since $D_{\mathrm{KL}}(\cdot \| \cdot)$ is minimized when its two arguments are equal, and $D_{\mathrm{KL}}(q_\theta \| m')$ is strictly convex in $q_\theta$ on the interior of the simplex, the unique minimum is achieved when $q_\theta = m'$. Substituting the definition of $m'$:
{\small
\begin{equation}
q_\theta = (1-\alpha) p + \alpha q_\theta \quad \Longleftrightarrow \quad (1-\alpha)(q_\theta - p) = 0.
\end{equation}}
For $\alpha < 1$, this again implies $q_\theta = p$. When $\alpha = 1$, SRKL reduces to a constant and the condition is trivially satisfied. Therefore, $q_\theta = p$ is the unique stationary point for SRKL, confirming convergence to the teacher distribution.
\end{proof}

The convergence analysis of SKL and SRKL demonstrates that both skewed variants share the same fixed point $q_\theta = p$ as forward and reverse KL, regardless of the skew coefficient $\alpha \in [0,1)$. This theoretical result supports the empirical findings in the main text that different loss geometries, including symmetric and asymmetric combinations, converge to similar performance levels when training is sufficiently long. The key insight is that while the optimization \emph{path} may differ (e.g., different convergence rates, different intermediate behaviors), the final \emph{destination} remains the same: matching the teacher distribution $p$.

\subsection{Detailed Derivation: Gradient of RKL with Respect to Student Logits}
\label{app:rkl_grad_sketch}
This subsection provides a step-by-step derivation of the gradient of reverse KL divergence with respect to student logits. This derivation complements the convergence analysis above and clarifies the mathematical structure underlying the RKL objective. We assume $q_\theta = \mathrm{softmax}(z^q)$ and treat the teacher distribution $p$ as fixed.

\noindent{\bf Step 1: Component form of RKL.}
The reverse KL divergence can be written in component form as:
{\small
\begin{equation}
\label{eq:rkl_component}
D_{\text{RKL}}(p, q_\theta) = \sum_{i=1}^{|\mathcal{V}|} q_\theta(i) \big( \log q_\theta(i) - \log p(i) \big),
\end{equation}} 
where $|\mathcal{V}|$ is the vocabulary size, and the summation is over all tokens in the vocabulary.

\noindent{\bf Step 2: Gradient with respect to $q_\theta$.}
To compute the gradient with respect to the student distribution $q_\theta$, we differentiate each term in the summation. Using the identity $\partial (q_i \log q_i)/\partial q_j = \delta_{ij} (1 + \log q_i)$, where $\delta_{ij}$ is the Kronecker delta, we obtain:
{\small
\begin{multline}
\label{eq:rkl_grad_q}
\frac{\partial \, D_{\text{RKL}}}{\partial \, q_\theta(j)} = 
\frac{\partial}{\partial q_\theta(j)} \sum_{i=1}^{|\mathcal{V}|} \big[ q_\theta(i) \log q_\theta(i) 
- q_\theta(i) \log p(i) \big] \\
= \log q_\theta(j) - \log p(j) + 1.
\end{multline}}
In vector form, this can be written as:
{\small
\begin{equation}
\label{eq:rkl_grad_vector}
\frac{\partial \, D_{\text{RKL}}}{\partial \, q_\theta} = \log q_\theta - \log p + \mathbf{1},
\end{equation}}
where $[\log q_\theta]_i = \log q_\theta(i)$, $[\log p]_i = \log p(i)$, and $\mathbf{1}$ is the all-ones vector of dimension $|\mathcal{V}|$.

\noindent{\bf Step 3: Chain rule through softmax.}
To obtain the gradient with respect to logits $z^q$, we apply the chain rule through the softmax function. The softmax Jacobian matrix is:
{\small
\begin{equation}
\label{eq:softmax_jacobian}
J_{\mathrm{softmax}}(z^q) = \mathrm{diag}(q_\theta) - q_\theta \, q_\theta^{\top},
\end{equation}}
where $\mathrm{diag}(q_\theta)$ is a diagonal matrix with $q_\theta$ on the diagonal. The $(i,j)$-th element of the Jacobian is:
{\small
\begin{equation}
[J_{\mathrm{softmax}}(z^q)]_{ij} = \frac{\partial q_\theta(i)}{\partial z^q(j)} = q_\theta(i) \left( \delta_{ij} - q_\theta(j) \right).
\end{equation}}
Applying the chain rule:
{\small
\begin{equation}
\label{eq:rkl_grad_logits_full}
\frac{\partial \, D_{\text{RKL}}}{\partial \, z^q} = J_{\mathrm{softmax}}(z^q) \, \frac{\partial \, D_{\text{RKL}}}{\partial \, q_\theta} = J_{\mathrm{softmax}}(z^q) \, \Big( \log q_\theta - \log p + \mathbf{1} \Big).
\end{equation}}

\noindent{\bf Step 4: Stationary point analysis.}
At the stationary point $q_\theta = p$ (elementwise equality), the term $(\log q_\theta - \log p + \mathbf{1})$ reduces to $\mathbf{1}$, since $\log q_\theta(i) - \log p(i) = 0$ for all $i$ when $q_\theta = p$. The softmax Jacobian satisfies the property $J_{\mathrm{softmax}}(z^q) \, \mathbf{1} = \mathbf{0}$, which follows from the constraint $\sum_{i=1}^{|\mathcal{V}|} q_\theta(i) = 1$ and the fact that:
{\small
\begin{multline}
\sum_{i=1}^{|\mathcal{V}|} [J_{\mathrm{softmax}}(z^q)]_{ij} = 
\sum_{i=1}^{|\mathcal{V}|} q_\theta(i) \big(\delta_{ij} - q_\theta(j)\big) \\
= q_\theta(j) - q_\theta(j) \sum_{i=1}^{|\mathcal{V}|} q_\theta(i) = 0.
\end{multline}}
Therefore, when $q_\theta = p$, the gradient in equation~\eqref{eq:rkl_grad_logits_full} is zero, confirming that $q_\theta = p$ is a stationary point. By the strict convexity of $D_{\text{RKL}}(p, q)$ in $q$ on the interior of the probability simplex, this stationary point is unique, establishing $q_\theta = p$ as the unique global minimum.


\end{document}